\documentclass{custom} 

\title[Gaussian mixture layers]{Gaussian mixture layers for neural networks}
\usepackage{algorithm}
\usepackage{algpseudocode}
\usepackage{amsmath, amssymb}
\usepackage{bbm}
\usepackage{caption}
\usepackage{enumitem}
\usepackage{graphicx}
\usepackage{listings}
\usepackage{mathrsfs}
\usepackage{mathtools}
\usepackage{microtype}
\usepackage{pgfplots}
\usepackage{tikz}
\usepackage{times}
\usepackage[Symbolsmallscale]{upgreek}
\usepackage{thm-restate} 
\usepackage{mdframed} 
\usepackage{euscript}
\usepackage{wrapfig}

\usepackage{custom}

\pgfplotsset{compat=1.17}
\usetikzlibrary{arrows.meta, calc, math, positioning, shapes.geometric}

\makeatletter
\def\set@curr@file#1{\def\@curr@file{#1}} 
\let\Ginclude@graphics\@org@Ginclude@graphics
\makeatother
\usepackage[load-configurations=version-1]{siunitx} 

\usepackage{float}

\hypersetup{citecolor=cyan, colorlinks=true, linkcolor=blue!75!black}

\coltauthor{%
\Name{Sinho Chewi} \Email{sinho.chewi@yale.edu} \\ \addr{Yale University} \AND
\Name{Philippe Rigollet} \Email{rigollet@mit.edu} \\ \addr{MIT} \AND
\Name{Yuling Yan} \Email{yuling.yan@wisc.edu} \\ \addr{University of Wisconsin-Madison}
}

\begin{document}

\maketitle

\begin{abstract}
The mean-field theory for two-layer neural networks considers infinitely wide networks that are linearly parameterized by a probability measure over the parameter space. This nonparametric perspective has significantly advanced both the theoretical and conceptual understanding of neural networks, with substantial efforts made to validate its applicability to networks of moderate width.
In this work, we explore the opposite direction, investigating whether dynamics can be directly implemented over probability measures. Specifically, we employ Gaussian mixture models as a flexible and expressive parametric family of distributions together with the theory of Wasserstein gradient flows to derive training dynamics for such measures. Our approach introduces a new type of layer---the Gaussian mixture (GM) layer---that can be integrated into neural network architectures.
As a proof of concept, we validate our proposal through experiments on simple classification tasks, where a GM layer achieves test performance comparable to that of a two-layer fully connected network. Furthermore, we examine the behavior of these dynamics and demonstrate numerically that GM layers exhibit markedly different behavior compared to classical fully connected layers, even when the latter are large enough to be considered in the mean-field regime.
\end{abstract}



\section{Introduction}\label{sec:intro}

Deep learning architectures are compositions of basic trainable \emph{layers}, and many major milestones of deep learning models can be traced back to innovations for these fundamental components.
For example, convolutional neural networks are based on \emph{convolution layers} and \emph{pooling layers}; residual neural networks are based on layers with \emph{skip connections}~\citep{He+16Res}; and more recently, the striking feats of large language models stem from the use of \emph{attention layers}~\citep{Vas+17Attention}.
Yet without guiding principles to map out the potential design space, the development of successful new layers has been largely elusive.

Whereas the vast majority of works on the theory of deep learning focus on understanding and explaining the behavior of existing architectures, in this paper we take the approach of applying the theory to propose new ones.
In particular,  we incorporate insights from the recent literature on the mean-field theory for neural networks~\citep{chizat2018global, MeiMonNgu18MeanField, MeiMisMon19MFNN, SirSpi20MeanFieldCLT, Chi22MFLangevin, NitWuSuz22MFLangevin, RotVan22NNInteracting} in order to propose a new type of layer, which we call the \emph{Gaussian mixture} (GM) \emph{layer}.

According to mean-field theory, which we briefly review in Section~\ref{sec:mf}, the training dynamics of a fully connected 2-layer neural network converges, under a natural scaling and in the limit of infinite width, to a Wasserstein gradient flow in the space of probability measures.
Here, the probability measure represents the \emph{distribution} of weights in the fully connected layer and thereby shifts our perspective from the evolution of individual neurons to the evolution of the collection thereof.
This framework has been successful at making nuanced and verifiable predictions about the training of certain simple neural networks~\citep{MeiMonNgu18MeanField, AbbAdsMis23LeapComplexity, BerMonZho24TimeScales}, but thus far the mean-field limit has mostly been used as a tool for analysis.
Indeed, it is not practical to actually implement the Wasserstein gradient flow, since doing so would require prohibitively large widths.

In this work, we explore the consequences of a \emph{prescriptive} mean-field rather than a descriptive one. This proposal results in manipulating infinite dimensional objects---distributions over $\R^d$--- and to implement it, we restrict the distribution over weights to a parametric \emph{Gaussian mixture} whose parameters are trained via standard optimization routines. The use of mixture modelling can also be motivated on grounds of clustering phenomena for neurons but we do not explore this question here.

Although our GM layer is proposed as a replacement for a wide fully connected layer, the parametrization and restriction of the distribution of neurons to the class of finite Gaussian mixtures leads to markedly different training dynamics.
We demonstrate this behavior through numerical experiments in Section~\ref{sec:exp} on the MNIST database, which also serve as a proof of concept for the incorporation of GM layers into neural network architecture design.
Our initial results are promising and show that a GM layer attains comparable test performance to a 2-layer fully connected network.
However, we stress that our goal is \emph{not} to demonstrate superiority of the GM layer over existing architectures which would necessitate going beyond the simple two-layer architecture for which the mean-field theory applies. As a result, while we demonstrate that the GM layer is modular and can be integrated into deep architectures, we leave a detailed investigation to future research as it would require substantial engineering developments that are beyond the scope of this proposal. 

\paragraph{Contributions.}
Our primary contribution is our Gaussian layer proposal, which we detail in Section~\ref{sec:gm_layer}.
We also show that whereas the Wasserstein gradient flow over empirical measures is implemented via the Euclidean gradient flow over the locations, the gradient flow over Gaussian mixtures---equipped with a certain natural geometry---is implemented via the Euclidean gradient flow over the means and the \emph{square roots} of the covariances (Theorem~\ref{thm:main_gmm}).

In Section~\ref{sec:implementation}, we suggest an efficient parametrization scheme to speed up implementation.

We conduct numerical experiments in Section~\ref{sec:exp} on the MNIST and Fashion-MNIST datasets, which serve as a proof of concept and provide some insights into the training dynamics of GM layers.
First, we show that GM layers can achieve comparable performance as a $2$-layer fully connected network.
This is despite the fact that the two exhibit quite different training dynamics (see Figure~\ref{fig:evolution}).

We also check that GM layers exhibit ``feature learning''---as is expected in the mean-field regime---in the sense that the distribution over first layer weights moves substantially away from initialization.
Finally, we exhibit performance gains obtained by going deeper (i.e., composing multiple GM layers).

\paragraph{Other related work.}
In recent years, Wasserstein gradient flows have been applied to numerous probabilistic problems, such as sampling and variational inference.
A common bottleneck for these applications is implementation of the flow, which can be achieved via stochastic dynamics in the case of sampling~\citep{JKO} but remains challenging in general.
The mean-field theory reviewed in Section~\ref{sec:mf} shows that the empirical distributions of the weights of a neural network along training indeed follow a Wasserstein gradient flow, but our goal in this work to maintain an evolution of a \emph{continuous} distribution over the weights.
The strategy we take here is to restrict the gradient flow to a parametric family, which was introduced in the context of variational inference for the family of Gaussians or mixtures of Gaussians~\citep{Lam+22GVI, Diao+23FBGVI} and later applied to filtering~\citep{LamBonBac23VarKushner} and mean-field variational inference~\citep{Gho+22MFVI, JiaChePoo23MFVI, Lac23IndepProj, YaoYan23MFVI}.

\section{Review of mean-field theory}\label{sec:mf}

We briefly review the mean-field theory for 2-layer fully connected neural networks.
A width-$m$ neural network computes a function $h_{\bs\omega,\bs\beta} :\R^d\to\R$ of the form
\begin{align*}
    h_{\bs\omega,\bs\beta}(x) = \frac{1}{m} \sum_{j=1}^m \omega_j \ReLU(\langle \beta_j, x\rangle)\,,
\end{align*}
where $\operatorname{\ReLU}\; : \R\to\R$ is the activation function---taken to be the ReLU function $\ReLU(z) = \max(0, z)$ for the rest of this paper---and $(\omega_j,\beta_j) \in \R\times\R^d$ are trainable weights, for each neuron $j\in [m]$.
Here and throughout, we use boldface to denote a collection of parameters.
The $1/m$ scaling above is characteristic of the mean-field regime and enable the following perspective. Let $\rho^{(m)}$ denote the empirical distribution of the weights, $\rho^{(m)} = m^{-1} \sum_{j=1}^m \delta_{(\omega_j,\beta_j)}$, then we can write $h_{\bs\omega,\bs\beta}(x) = \int \omega \ReLU(\langle \beta,x\rangle)\,\rho^{(m)}(\D \omega,\D \beta)$.
In this formulation, however, we can make sense of this expression even when $\rho^{(m)}$ is no longer an empirical measure, and we can view $h$ as being parameterized by a probability measure $\rho$:
\begin{align}\label{eq:hrho}
    h_\rho(x) = \int \omega\ReLU(\langle \beta,x\rangle)\,\rho(\D \omega,\D\beta)\,.
\end{align}

Consider a loss\footnote{The loss typically depends on a training set, but this is irrelevant for our discussion here so it is omitted.} objective $\ms L$ and minimize the objective $\ms L(h_{\bs\omega,\bs\beta})$ by following the gradient flow for the weights ${\{\omega_j,\beta_j\}}_{j\in [m]}$.
This gives rise to a curve of parameters ${(\bs\omega(t), \bs\beta(t))}_{t\ge 0}$, and corresponding empirical measures $\rho^{(m)}(t) = m^{-1} \sum_{j=1}^m \delta_{(\omega_j(t),\beta_j(t))}$.
The insight of mean-field theory is that the evolution of the empirical measures can be described as the (time-rescaled) gradient flow of the loss function $\rho \mapsto \ms L(h_\rho)$ over the space of probability measures, equipped with the Wasserstein metric from optimal transport, initialized at $\rho^{(m)}(0)$.
We refer to~\citet{Vil03Topics, AGS, Vil09OT, San15OT} for background on optimal transport and Wasserstein gradient flows.

The advantage of this reformulation is that it admits a well-defined limit as $m\to\infty$: if $\rho^{(m)}(0) \to \rho(0)$, then the curve of measures ${(\rho^{(m)}(t))}_{t\ge 0}$ converges to the Wasserstein gradient flow of $\rho\mapsto \ms L(h_\rho)$, initialized at $\rho(0)$.
This mean-field limit is, in some cases, easier to study than the original dynamics over the weights, and leads to predictions about the behavior of wide neural networks.

To summarize: the training dynamics of a finite-width neural network correspond to a Wasserstein gradient flow, initialized at (and remaining through its trajectory) an empirical measure, but the Wasserstein gradient flow picture is more general because it allows for flows of continuous measures.
Unfortunately, in the latter case, the Wasserstein gradient flow does not readily lend itself to tractable implementation.
The mean-field theory described above shows that it is \emph{well-approximated} by a gradient flow started at an empirical measure, but this approximation often requires prohibitively large width.
In the next section, we take the familiar approach from statistics of restricting the measures to a parametric family, namely, the set of finite Gaussian mixtures.

\section{A mean-field theory over the space of Gaussian mixtures}\label{sec:gm_layer}

We now introduce the Gaussian Mixture (GM) layer, beginning with the case of a single GM layer (corresponding with the mean-field theory described in Section~\ref{sec:mf}).
Here, we restrict the measure $\rho$ in~\eqref{eq:hrho} to be a Gaussian mixture with $K$ components:
\begin{align}\label{eq:gmm}
    \rho = \rho_{\bs\mu,\bs\Sigma} \deq \frac{1}{K} \sum_{k=1}^K \mc N(\mu_k,\Sigma_k)\,, \qquad \mu_k \in \R^{d+1}\,, \qquad \Sigma_k \in \R^{(d+1)\times (d+1)}\,.
\end{align}
Thus, the measure $\rho$ is now parameterized by a set of means and covariances for the components of the Gaussian mixture.
We use the short-hand notation $h_{\bs\mu,\bs\Sigma} \deq h_{\rho_{\bs\mu,\bs\Sigma}}$.
With this restriction, we can now train the GM layer by minimizing $\ms L(h_{\bs\mu,\bs\Sigma})$ with respect to the parameters $(\bs\mu,\bs\Sigma)$.
For example, in a regression task with a labeled dataset $\{x_i,y_i\}_{i\in [n]}$, we might take the squared loss defined by $\ms L(h) \deq \sum_{i=1}^n (y_i - h(x_i))^2$.

The use of mixture modelling to model the distribution over neurons is motivated by the empirical observation that for many problems, the neurons tend to \emph{cluster}~\citep[e.g.,][]{PapHanDon20NeuralCollapse, Che+23StochasticCollapse}.
Given this ansatz, the use of Gaussian mixtures emerges as a natural model for $\rho$, although other alternatives could be explored.

As discussed in Section~\ref{sec:mf}, it is well-known from mean-field theory that the Wasserstein gradient flow restricted to empirical measures is implemented, up to time rescaling, by the Euclidean gradient flow with respect to the locations of the particles.
When we move to Gaussian mixtures, we effectively replace the particles $\theta_j = (\omega_j, \beta_j)$ with ``Gaussian particles'' $\mc N(\mu_k,\Sigma_k)$, which are themselves distributions over $\theta$ but can be viewed simply as $(\mu_k,\Sigma_k)$ pairs.
In the case $K=1$ of a single Gaussian particle, it turns out that the Wasserstein gradient flow restricted to Gaussian measures is implemented simply by
evolving the parameters $(\mu,C)$ via the (Euclidean) gradient flow, where $\Sigma = CC^\top$.
This is usually known as the \emph{Bures--Wasserstein gradient flow} (see Appendix~\ref{app:interpretation}).

In the case $K\ge 2$, it is no longer possible to follow the Wasserstein gradient flow restricted to Gaussian mixtures.
Indeed, the latter is not explicit, since the space of Gaussian mixtures is not a geodesically convex subset of the Wasserstein space.
However, we can instead follow a Wasserstein gradient flow for the \emph{mixing measure} $\nu \deq \frac{1}{K} \sum_{k=1}^K \delta_{(\mu_k,\Sigma_k)}$ over the Bures--Wasserstein space, see Appendix~\ref{app:interpretation} for details.
This geometric structure was previously introduced in~\citet{CheGeoTan19GMM, DelDes20GMM} and further exploited in~\cite{Lam+22GVI} in the context of variational inference.
We prove that this flow does admit a simple implementation.

\begin{theorem}[Informal]\label{thm:main_gmm}
    Let $\loss$ be a loss function over the space of probability measures.
    Then, the Gaussian mixture gradient flow for $\loss$ is equivalent to the Euclidean gradient flow of the objective $\loss(h_{\bs\mu,\bs\Sigma})$ with respect to the parameters $(\bs\mu,\bs C)$, where $\Sigma_k = C_k C_k^\top$ for each $k\in [K]$.
\end{theorem}

In other words, compared to ordinary neural networks, training GM layers is accomplished by simply incorporating gradient steps for the \emph{square roots} of the covariances.

\section{Implementation}\label{sec:implementation}

In this section, we enhance the flexibility and tractability of GM layers by allowing for vector-valued outputs (needed for multi-class classification), reduced parametrization, and composition.

\subsection{Multi-class classification and vector-valued outputs}

Consider a multi-class classification problem with $L+1$ labels denoted $\{0, \dots, L\}$.  Suppose we are given a dataset ${\{x_i,y_i\}}_{i\in [n]}$, where each $x_i\in\R^d$ and $y_i \in \{0,1,\dotsc,L\}$.
It suffices to describe how to parameterize a vector-valued function $h : \R^d\to\R^{L}$, since we can then apply the logistic loss 
$$
\ms L(h) \deq -\sum_{i=1}^n \Bigl\{\sum_{\ell=1}^{L} h(x_i)_\ell \ind_{y_i = \ell} - \log\bigl(1+\sum_{\ell=1}^{L} \exp(h(x_i)_\ell)\bigr)\Bigr\}\,.
$$

The most straightforward way to parameterize the function $h: \R^d \to \R^{L}$ is via~\eqref{eq:hrho} and~\eqref{eq:gmm}, where now $(\omega,\beta) \in \R^{L}\times\R^d$, i.e., the Gaussian mixture $\rho_{\bs\mu,\bs\Sigma}$ is a distribution over $\R^{d+L}$.
In other words,
\begin{align}\label{eq:hgmm}
    h_{\bs\mu,\bs\Sigma}(x) = K^{-1} \sum_{k\in [K]} \E_{(\omega,\beta)\sim\mc N(\mu_k,\Sigma_k)}[\omega\ReLU(\langle \beta,x\rangle)]\,.
\end{align}
However, the number of parameters becomes  $\Theta((d+L)^2\,K)$, which is prohibitively large.

\subsection{Reducing the number of parameters}\label{subsec:reducing_param}

 To reduce the number of parameters, we propose to incorporate sparsity into the model by considering only diagonal covariance matrices for $\beta$: $\beta\sim\mathcal{N}(\mu^\beta,\diag (\sigma^2))$ where $\diag (\sigma^2)$ is a diagonal matrix whose entries are given by the vector $\sigma^2=\sigma\odot \sigma \in \R^d$. Moreover, for jointly Gaussian $(\omega,\beta)$, the conditional mean of $\omega$ given $\beta$ is affine:  $\E[\omega\mid\beta] = U\beta+v$. Since~\eqref{eq:hgmm} only requires to know the conditional expectation $\E[\omega\mid\beta]$, we model each component of the Gaussian mixture as
\[
\beta\sim\mathcal{N}(\mu^\beta,\diag (\sigma^2))\,,\qquad \E[\omega\mid\beta] = U\beta+v\,.
\]
The trainable parameters for such a component are $\mu^\beta \in \R^d$, $\sigma \in \R^d$, $U \in \R^{L\times d}$, and $v\in\R^{L}$.
Hence, the parameters for the full Gaussian mixture are $\bs\theta \deq (\bs\mu^\beta,\bs\sigma,\bs U, \bs v) = \{(\mu^\beta_k, \sigma_k, U_k, v_k)\}_{k\in [K]}$.
This leads to
\begin{align}\label{eq:hcheaper}
    h_{\bs\theta}(x) = K^{-1} \sum_{k\in [K]} \E_{\beta\sim \mc N(\mu_k^\beta, \diag(\sigma_k^2))}[(U\beta + v)\,\ReLU(\langle \beta,x\rangle)]\,.
\end{align}
This use of~\eqref{eq:hcheaper} cuts down the number of parameters to $\Theta(dKL)$, a substantial savings when $L \ll d$ (as is typical).
An additional benefit is that parametrization by $\bs\sigma$ automatically maintains the positive semidefiniteness of the covariances during training, without recourse to costly projection steps.

To summarize, for multi-class classification, we train the parameters $\bs\theta$ via an optimization algorithm (in our experiments, we use vanilla stochastic gradient descent) on the objective $\ms L(h_{\bs\theta})$, where $\ms L$ is the multi-class logistic loss and $h_{\bs\theta}$ is given in~\eqref{eq:hcheaper}.

\subsection{Stacking GM layers}

In the previous subsection, we showed how to parameterize a function $h : \R^d\to\R^{L}$ as a GM layer.
Since the input and output dimensions are arbitrary, this construction can be readily composed with other types of layers---including GM layers themselves---in order to build up deep neural network architectures. Indeed, as depicted in Figure~\ref{fig:gmlayer}, the GM layer can be dropped in as a replacement for an (infinitely wide) fully connected layer. 

We leave the question of designing and optimizing deep architectures that integrate the GM layer to future research. Instead, we focus on a single GM layer, and call it a \emph{GM network}.

\begin{figure}[ht]
    \centering
\begin{tikzpicture}[
    neuron/.style={ellipse, draw, thick, fill=blue!20, minimum width=1cm, minimum height=1cm, inner sep=1mm, align=center},
    layer/.style={draw, thick, rectangle, fill=red!20, minimum height=1cm, minimum width=1cm, inner sep=1mm, align=center},
    arrow/.style={->,>={Latex[length=3mm,width=2mm]},thick},
    node distance=0.8cm and 1cm, 
    every node/.append style={transform shape},
    auto, scale=0.9
]

\node[neuron] (input1) {$x$};

\node[neuron, below=of input1] (hidden1) {$\cdots$};
\node[neuron, left=of hidden1] (hidden2) {$\ReLU(\beta_1^\top x)$};
\node[neuron, right=of hidden1] (hidden3) {$\ReLU(\beta_m^\top x)$};

\node[neuron, below=of hidden2] (output1) {$h_1(x)$};
\node[neuron, below=of hidden1] (output2) {$\cdots$};
\node[neuron, below=of hidden3] (output3) {$h_{L}(x)$};

\draw[arrow] (input1) -- (hidden1);
\draw[arrow] (input1) -- node[swap] {$\beta_1$} (hidden2);
\draw[arrow] (input1) -- node {$\beta_m$} (hidden3);

\draw[arrow] (hidden2) -- node[swap] {$\omega_{1,1}$} (output1);
\draw[arrow] (hidden3) -- (output1);
\draw[arrow] (hidden1) -- (output1);
\draw[arrow] (hidden2) -- (output3);
\draw[arrow] (hidden3) -- node {$\omega_{L,m}$} (output3);
\draw[arrow] (hidden1) -- (output3);
\draw[arrow] (hidden1) -- (output2);
\draw[arrow] (hidden2) -- (output2);
\draw[arrow] (hidden3) -- (output2);

\node[layer, right=of output3, xshift = 2mm] (output) {$\int \omega_\ell \ReLU(\beta^\top x ) \, \rho(\D \omega,\D\beta)$ \quad $(1\leq \ell \leq L)$};

\node[layer,above=of output, yshift = -1.5mm] (density) {
    \begin{tikzpicture}
        \begin{axis}[
            axis lines=middle,
            xtick=\empty,
            ytick=\empty,
            width = 4cm,
            height = 2.5cm,
            grid=major,
            domain=-10:10,
            samples=100,
            smooth,
            no markers,
            thick
        ]
        \addplot[smooth, no markers, blue!50!black] {
        + (0.5)*exp(-(x-3)^2/(2*2))/sqrt(2*pi*2) 
        + (0.5)*exp(-(x+4)^2/(2*4))/sqrt(2*pi*4) 
        };
        \end{axis}
    \end{tikzpicture}
};

\node[layer, above=of density, yshift = -1mm] (input2) {$x$};
\node[align=center, left= of density, xshift = 0.8cm] {$\rho$};

\draw[arrow] (input2) -- (density);
\draw[arrow] (density) -- (output);

\end{tikzpicture}
\medskip
\caption{A GM layer (right) can act as a replacement for a fully connected layer (left).}\label{fig:gmlayer}
\end{figure}
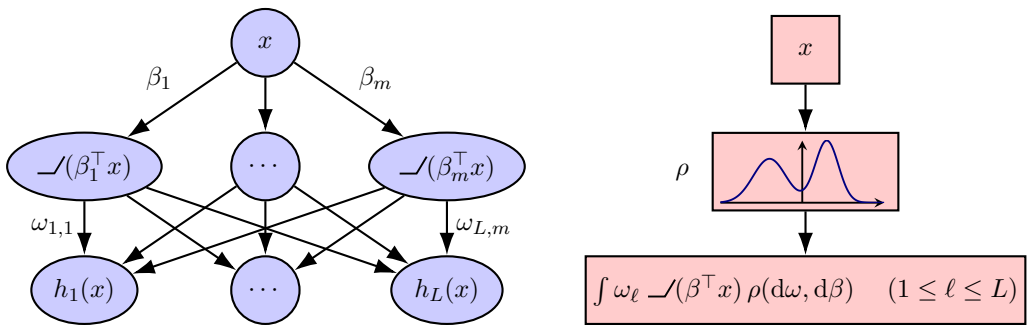

\section{Experiments}\label{sec:exp}

\paragraph{Dataset.} We test the performance of neural networks with GM layers on multi-class classification on two widely used datasets: MNIST \citep{lecun-mnisthandwrittendigit-2010} and Fashion-MNIST \citep{xiao2017fashion}. Both datasets consist of 60,000 training examples and 10,000 test examples, where each example is a $28 \times 28$ grayscale image, associated with a label from one of 10 classes. Each image is vectorized and normalized to have zero mean and unit standard deviation. Throughout this section, test error refers to the misclassification error evaluated over the test set.

\paragraph{Setup.} The number of components $K$ is a hyperparameter: larger $K$ enables more expressive GM layers, while smaller $K$ speeds up computation. We consider $K \in \{5,10,20\}$ for different experiments. For a GM layer with parameters $\mu^{\beta}$, $\sigma$, $U$, and $v$, we initialize the entries of $\mu^{\beta}$, $U$ and $v$ i.i.d.~from $\mathcal{N}(0,\gamma^2)$, and the entries of $\sigma$ all equal to $\gamma$, for some $\gamma>0$. For most of the experiments we fix $\gamma=1/2$ unless otherwise mentioned. We train the network using SGD with batch size $64$ and fixed learning rate $1$ for the parameter $\sigma$ and $0.1$ for all other parameters.

\begin{figure}
\centering
\includegraphics[scale=0.6]{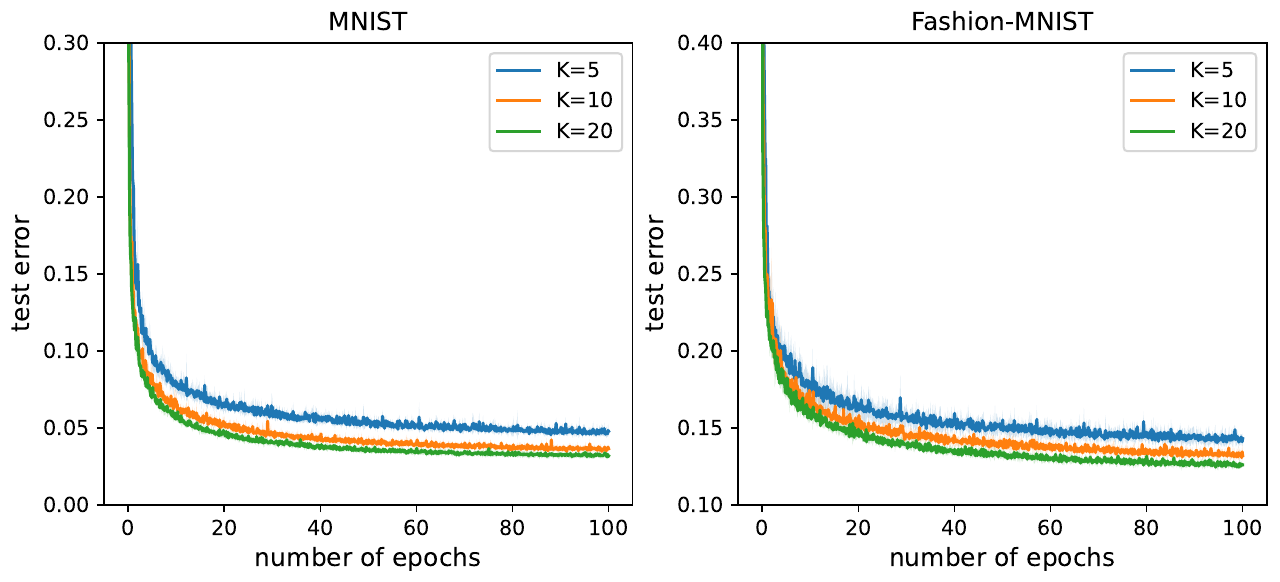} 
\caption{Test error (with error bars) for a GM Network with $K=5,10,20$ number of components vs.~the number of epochs. The left (resp.~right) panel shows the result for MNIST (resp.~Fashion-MNIST) dataset. The error bars are computed over 5 independent trials.}\label{fig:GMM-test-error}
\end{figure}

\begin{figure}[t]
\centering
\includegraphics[scale=0.6]{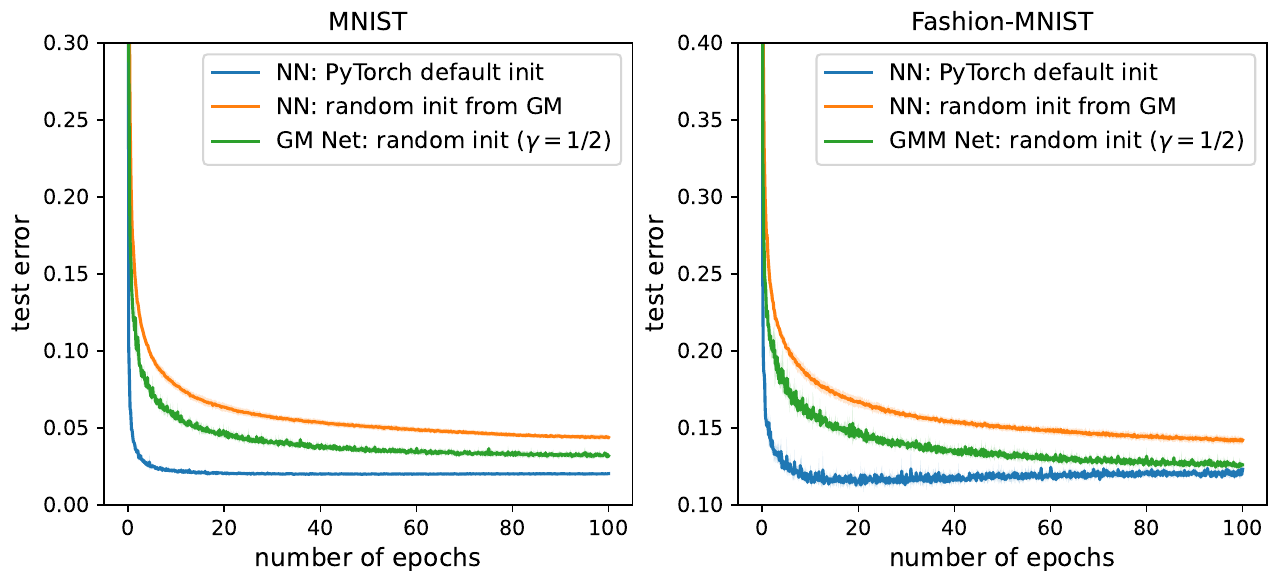} 
\caption{Test error (with error bars) for a 2-layer fully connected network using two different initialization schemes vs.~the number of epochs. The left (resp.~right) panel shows the results for the MNIST (resp.~Fashion-MNIST) dataset. The error bars are computed over 5 independent trials.}\label{fig:NN-test-error}
\end{figure}

\paragraph{Test error.} For both datasets, we test the performance of a network with one GM layer, where the number of components $K$ takes values in $\{5,10,20\}$. The results are presented in Figure~\ref{fig:GMM-test-error}. As we can see, a single GM layer with 20 components achieves a test error of $\approx 2.77\%$ for MNIST, and $\approx 12.13\%$ for Fashion-MNIST. Increasing $K$ leads to better performance, but the marginal improvement is very small after $K$ exceeds $10$, suggesting that in practice one may choose $K$ between 10 and 20 to strike a balance between expressive power and computational efficiency.

We also compare the test error of the GM network with the 2-layer fully connected network. We set the width of the latter network  $m$ to be $1000$ (in fact the test error curve is almost the same for $m\in\{100,500,1000\}$, so we just present the result for $m=1000$ for simplicity). We use two different methods to initialize the parameters of a fully-connected layer: PyTorch's default method (Kaiming uniform \citep{he2015delving}), and random initialization drawn i.i.d from the initial distribution we used in training the GM network with $K=20$ components (for a fair comparison with the GM network). The results are presented in Figure~\ref{fig:NN-test-error}. We can see that although in the end all curves converges to comparable test errors, the number of epochs required to achieve low test are different: the GM network performs better than fully connected network with random initialization, but worse than that with PyTorch's default initialization. This observation calls for future investigation on better initialization schemes for GM layers. 

\begin{figure}[t]
\centering
\includegraphics[scale=0.65]{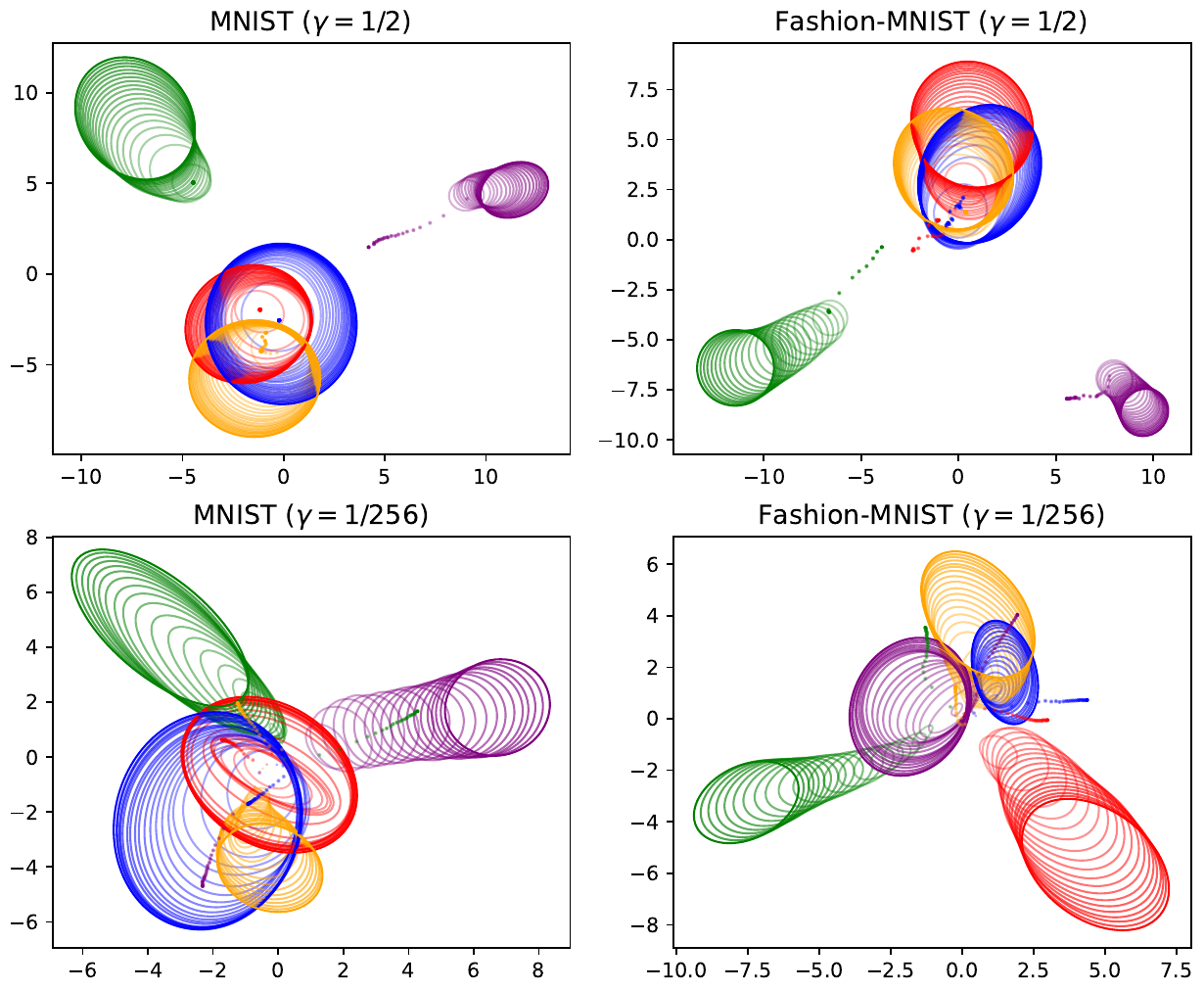} 
\caption{The evolution of the Gaussian components (marginalized over $\beta$) and weights $\beta_j$ of the neurons, projected onto the top 2 PCs of the final GM distribution, across 200 epochs of training. The number of components for the GM net (resp.  number of neurons in the 2-layer neural network) is $K=5$. The projected Gaussian components are represented by their covariance ellipses centered at their means, while the projected weights of the neurons are depicted as dots. We use the same color for the evolution of the same Gaussian component and neuron, with increasing opacity as the number of epochs increases. The left (resp.~right) plots show results for MNIST (resp.~Fashion-MNIST), while the top (resp.~bottom) plots use initialization scale $\gamma=1/2$ (resp.~$\gamma=1/256$).}
\label{fig:evolution}
\end{figure}

\paragraph{Evolution of Gaussian components.} In order to visualize the evolution of the Gaussian mixture $(\rho_t)_{t\geq0}$ during the training phase, we train the network for $T=200$ epochs, compute the top 2 principal subspace of the final Gaussian mixture distribution $\rho_T$ (marginalized over $\beta$), and project the entire training trajectory $(\rho_t)_{0\leq t \leq T}$ (marginalized over $\beta$) onto this 2-dimensional subspace. 
Figure~\ref{fig:evolution} depicts the evolution of this 2-dimensional distribution across 200 epochs for $K=5$, using two types of initialization schemes $\gamma\in\{1/2,1/256\}$. As we can see, the means of the five Gaussian components move far away from their initializations (even when they are initialized near zero, which is the case when $\gamma=1/256$), and the covariance matrices also become non-isotropic quickly. This also shows that the training dynamics of networks with GM layers are not sensitive to initialization. 

We also train a fully-connected 2-layer neural network with width $m=K$ to compare the training dynamics of networks with a GM layer and a fully-connected layer. To make the comparison fair, we initialize the neurons at $\beta_k=\mu_k$ and $\omega_k = U\beta_k + v$ for $1\leq k \leq K$ where $\{\mu_k\}_{1\leq k \leq K}$, $U$ and $v$ are the parameters of the initial Gaussian mixture distribution. We train the neural network using the same SGD algorithm with learning rate $0.1$. The evolution of these neurons (projected onto the same 2-dimensional subspace) across 200 epochs is also shown in Figure~\ref{fig:evolution}. We can see that they exhibit drastically different training dynamics from that of the GM network. For example, when initialized at scale $\gamma=1/2$, the Gaussian components of the GM layer tend to move far away from zero, while the neurons of the fully-connected layer do not exhibit this trend.

\begin{figure}[t]
\centering
\includegraphics[scale=0.65]{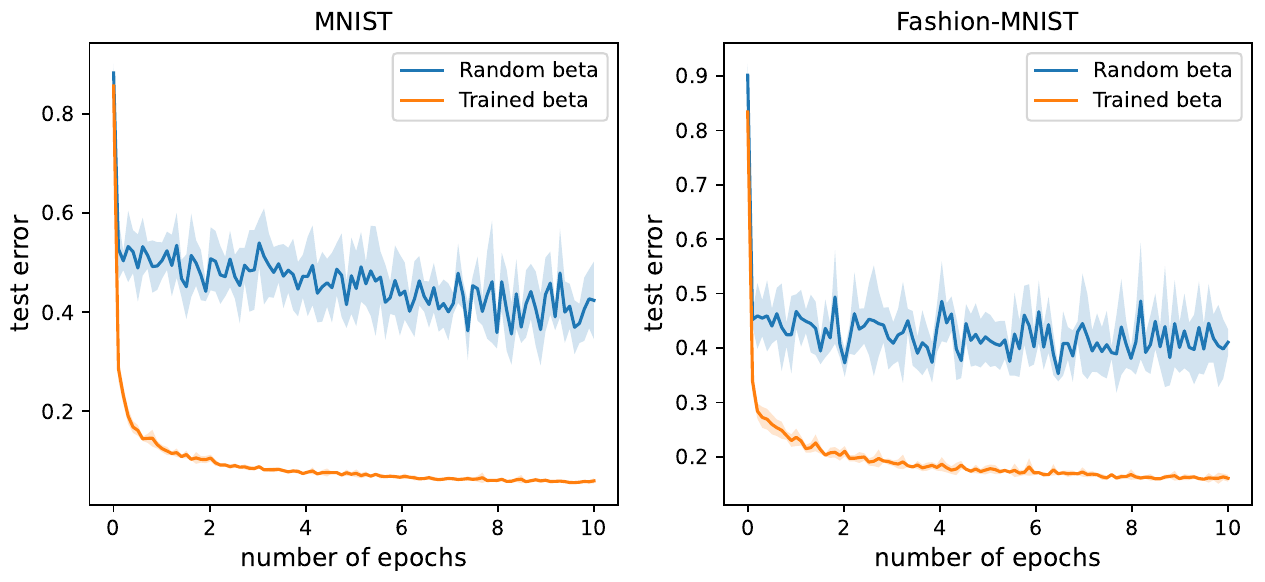} 
\caption{Test error (with error bars) for training with or without updating the marginal distribution over $\beta$ vs.~the number of epochs. The left (resp.~right) panel shows the result for the MNIST (resp.~Fashion-MNIST) dataset. The error bars are computed over 5 independent trials.}
\label{fig:fix_beta}
\end{figure}

\paragraph{Mean field vs.~``NTK'' regime.} We design a simple experiment to inspect whether the training of a network with a single GM layer is in the neural tangent kernel (NTK) a.k.a ``lazy training'' regime~\citep{Jac+18NTK, ChiOyaBac19Lazy, Du+19GDNN, BarMonRak21DeepLearning}, or the mean-field regime.
Indeed, since the presence of feature learning is a powerful motivation for the mean-field regime, it is important to check that this intuition also carries over to GM layers.
Although lazy training is not formally defined for GM layers, we can loosely take it to be the case when the distribution over the ``first layer weights'' $\beta$ does not significantly move away from its initialization.

We fix the marginal distribution over $\beta$ at its initialization (which can be achieved by setting the learning rates for $\mu^{\beta}$ and $\sigma$ to be $0$) and only update $U$ and $v$ for each Gaussian component. Figure~\ref{fig:fix_beta} compares the performance of training the network with fixed $\beta$ vs.\ the network with trained $\beta$. If we fix the marginal distribution over $\beta$, the trained network can only achieve a test error of $\approx 40\%$ for both the MNIST and Fashion-MNIST datasets.
This suggests that ``feature learning'' is indeed crucial for the performance of GM layers.

\begin{figure}[t]
\centering
\includegraphics[scale=0.6]{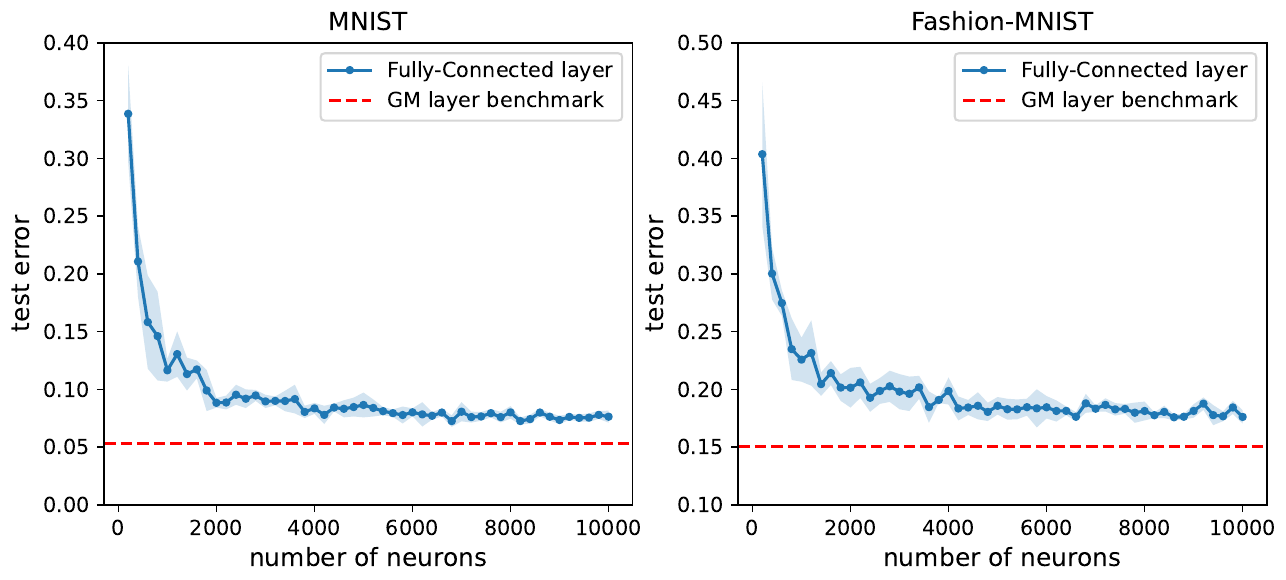} 
\caption{Test error (with error bars) for fully connected 2-layer networks constructed via subsampling vs.~width. The neurons are sampled from a trained GM layer, whose test error is also plotted as a benchmark. The left (resp.~right) panel shows the result for the MNIST (resp.~Fashion-MNIST) dataset. The error bars are computed over 5 independent trials.}\label{fig:subsampling}
\end{figure}

\paragraph{Monte Carlo reduction to fully connected layers.} After training a network with a GM layer, which gives a distribution jointly over $(\omega,\beta)$, one natural question is whether it is possible to construct a fully connected 2-layer network with similar performance by sampling a reasonable number of neurons from this Gaussian mixture distribution. To answer this question, we first train a network with a GM layer with $K=20$ components and $20$ epochs on MNIST. Then we construct fully connected 2-layer networks by sampling $m$ neurons (i.e., $(\omega,\beta)$ pairs) from the trained Gaussian mixture distribution, and evaluate their test errors without training. The results reported in in Figure~\ref{fig:subsampling} are indicating of the classical $1/\sqrt{m}$ convergence rate of Monte Carlo approximation. Unfortunately, this convergence is to slow from the perspective of test error: even for $m=10^4$ there is still a gap between the performance of the GM network and its Monte Carlo approximation. This is a consequence of the high-dimensional nature of the space of parameter: we sample, on average $10^4/20=500$ points per component, and this is not sufficient to estimate accurately a Gaussian integral in dimension $28\times 28=784$.
 
\begin{figure}[t]
\centering
\includegraphics[scale=0.65]{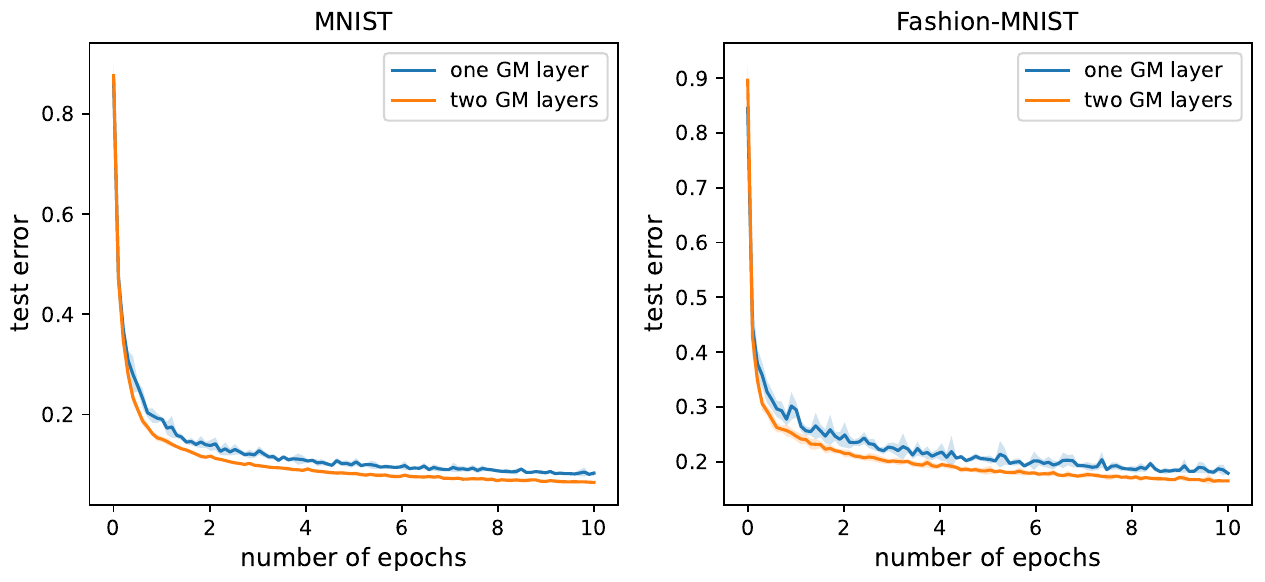} 
\caption{Test error (with error bars) for networks with one or two GM layers vs.~the number of epochs. The left (resp.~right) panel shows the result for the MNIST (resp.~Fashion-MNIST) dataset. The error bars are computed over 5 independent trials.}
\label{fig:two-layers}
\end{figure}

\paragraph{Stacking multiple GM layers.} In the previous experiments, we used networks with a single GM layer. In this experiment, we demonstrate that stacking GM layers can improve classification performance. Specifically, we stack a GM layer with an input dimension of 784 and an output dimension of \( L' = 100 \) with a second GM layer with an input dimension of \( L' = 100 \) and an output dimension of \( L = 9 \), both layers utilizing \( K = 10 \) Gaussian components. This configuration yields a lower test error compared to a single GM layer with the same number of Gaussian components, as shown in Figure~\ref{fig:two-layers}. Our primary objective is to provide a straightforward example illustrating the benefits of stacking multiple GM layers. We have not optimized hyperparameters such as the number of components \( K \) and the width of the middle layer \( m \) for optimal performance. Additionally, constructing a deeper architecture to achieve competitive performance on classic benchmark tasks like image classification would necessitate the integration and optimization of convolutional layers. This would shift the focus to secondary aspects that are beyond the scope of this paper.

\section{Conclusion}

In this paper, we have introduced GM layers as a novel layer type for neural network architectures, offering a fresh perspective that bridges concepts from mean-field theory and variational inference. This  approach opens up a wealth of unexplored possibilities for layer design, suggesting new avenues for incorporating diverse ``variational families'' to model $\rho$ beyond Gaussian mixtures, and for optimizing with respect to various geometries within the variational family.
For instance, while our current model restricts to Gaussian mixtures with \emph{equal weights}, this limitation can be addressed by optimizing over Gaussian mixtures with unequal weights via the Wasserstein--Fisher--Rao geometry~\citep{LieMieSav16HK, Chi+18FR, LieMieSav18HK,LuLuNol19,yan2023learning}, as demonstrated in~\citet[Appendix H]{Lam+22GVI}. Such advancements may also inspire alternatives for other types of layers, including convolutional and attention layers.

A notable limitation of our work is the preliminary nature of our experiments. Further empirical investigation is required to refine our design choices and compare them with existing architectures. For example, the sparse parametrization in Subsection~\ref{subsec:reducing_param} is quite drastic, suggesting the need to explore alternatives like low-rank factorization of the covariance matrix. Moreover, developing effective initialization and training strategies for GM layers remains an open question. We anticipate that addressing these challenges will pave the way for more robust and versatile neural network architectures in the future.

\subsection*{Acknowledgments}

SC was supported by the Eric and Wendy Schmidt Fund at the Institute for
Advanced Study.
PR was supported by NSF grants DMS-2022448 and CCF-2106377\@.
YY was supported in part by a Norbert Wiener Postdoctoral Fellowship from MIT.

\newpage

\bibliography{bibfileWGF}

\begin{thebibliography}{42}
\providecommand{\natexlab}[1]{#1}
\providecommand{\url}[1]{\texttt{#1}}
\expandafter\ifx\csname urlstyle\endcsname\relax
  \providecommand{\doi}[1]{doi: #1}\else
  \providecommand{\doi}{doi: \begingroup \urlstyle{rm}\Url}\fi

\bibitem[Abbe et~al.(2023)Abbe, Boix-Adser{\`a}, and Misiakiewicz]{AbbAdsMis23LeapComplexity}
Emmanuel Abbe, Enric Boix-Adser{\`a}, and Theodor Misiakiewicz.
\newblock {SGD} learning on neural networks: leap complexity and saddle-to-saddle dynamics.
\newblock In Gergely Neu and Lorenzo Rosasco, editors, \emph{Proceedings of Thirty Sixth Conference on Learning Theory}, volume 195 of \emph{Proceedings of Machine Learning Research}, pages 2552--2623. PMLR, 7 2023.

\bibitem[Altschuler et~al.(2021)Altschuler, Chewi, Gerber, and Stromme]{altschuler2021averaging}
Jason~M. Altschuler, Sinho Chewi, Patrik~R. Gerber, and Austin~J. Stromme.
\newblock Averaging on the {B}ures--{W}asserstein manifold: dimension-free convergence of gradient descent.
\newblock \emph{Advances in Neural Information Processing Systems}, 34:\penalty0 22132--22145, 2021.

\bibitem[Ambrosio et~al.(2008)Ambrosio, Gigli, and Savar\'{e}]{AGS}
Luigi Ambrosio, Nicola Gigli, and Giuseppe Savar\'{e}.
\newblock \emph{Gradient flows in metric spaces and in the space of probability measures}.
\newblock Lectures in Mathematics ETH Z\"{u}rich. Birkh\"{a}user Verlag, Basel, second edition, 2008.

\bibitem[Bartlett et~al.(2021)Bartlett, Montanari, and Rakhlin]{BarMonRak21DeepLearning}
Peter~L. Bartlett, Andrea Montanari, and Alexander Rakhlin.
\newblock Deep learning: a statistical viewpoint.
\newblock \emph{Acta Numer.}, 30:\penalty0 87--201, 2021.

\bibitem[Berthier et~al.(2024)Berthier, Montanari, and Zhou]{BerMonZho24TimeScales}
Rapha\"{e}l Berthier, Andrea Montanari, and Kangjie Zhou.
\newblock Learning time-scales in two-layers neural networks.
\newblock \emph{arXiv preprint 2303.00055}, 2024.

\bibitem[Bhatia et~al.(2019)Bhatia, Jain, and Lim]{BhaJaiLim19BW}
Rajendra Bhatia, Tanvi Jain, and Yongdo Lim.
\newblock On the {B}ures--{W}asserstein distance between positive definite matrices.
\newblock \emph{Expo. Math.}, 37\penalty0 (2):\penalty0 165--191, 2019.

\bibitem[Burer and Monteiro(2003)]{BurMon03LowRank}
Samuel Burer and Renato D.~C. Monteiro.
\newblock A nonlinear programming algorithm for solving semidefinite programs via low-rank factorization.
\newblock \emph{Math. Program.}, 95\penalty0 (2):\penalty0 329--357, 2003.
\newblock Computational semidefinite and second order cone programming: the state of the art.

\bibitem[Chen et~al.(2023)Chen, Kunin, Yamamura, and Ganguli]{Che+23StochasticCollapse}
Feng Chen, Daniel Kunin, Atsushi Yamamura, and Surya Ganguli.
\newblock Stochastic collapse: how gradient noise attracts {SGD} dynamics towards simpler subnetworks.
\newblock In A.~Oh, T.~Naumann, A.~Globerson, K.~Saenko, M.~Hardt, and S.~Levine, editors, \emph{Advances in Neural Information Processing Systems}, volume~36, pages 35027--35063. Curran Associates, Inc., 2023.

\bibitem[Chen et~al.(2019)Chen, Georgiou, and Tannenbaum]{CheGeoTan19GMM}
Yongxin Chen, Tryphon~T. Georgiou, and Allen Tannenbaum.
\newblock Optimal transport for {G}aussian mixture models.
\newblock \emph{IEEE Access}, 7:\penalty0 6269--6278, 2019.

\bibitem[Chizat(2022)]{Chi22MFLangevin}
L{\'e}na{\"\i}c Chizat.
\newblock Mean-field {L}angevin dynamics: exponential convergence and annealing.
\newblock \emph{Transactions on Machine Learning Research}, 2022.

\bibitem[Chizat and Bach(2018)]{chizat2018global}
L\'{e}na\"{\i}c Chizat and Francis Bach.
\newblock On the global convergence of gradient descent for over-parameterized models using optimal transport.
\newblock \emph{Advances in Neural Information Processing Systems}, 31, 2018.

\bibitem[Chizat et~al.(2018)Chizat, Peyr\'{e}, Schmitzer, and Vialard]{Chi+18FR}
L\'{e}na\"{\i}c Chizat, Gabriel Peyr\'{e}, Bernhard Schmitzer, and Fran\c{c}ois-Xavier Vialard.
\newblock An interpolating distance between optimal transport and {F}isher--{R}ao metrics.
\newblock \emph{Found. Comput. Math.}, 18\penalty0 (1):\penalty0 1--44, 2018.

\bibitem[Chizat et~al.(2019)Chizat, Oyallon, and Bach]{ChiOyaBac19Lazy}
L\'{e}na\"{\i}c Chizat, Edouard Oyallon, and Francis Bach.
\newblock On lazy training in differentiable programming.
\newblock In H.~Wallach, H.~Larochelle, A.~Beygelzimer, F.~d\textquotesingle Alch\'{e}-Buc, E.~Fox, and R.~Garnett, editors, \emph{Advances in Neural Information Processing Systems}, volume~32. Curran Associates, Inc., 2019.

\bibitem[Delon and Desolneux(2020)]{DelDes20GMM}
Julie Delon and Agn\`es Desolneux.
\newblock A {W}asserstein-type distance in the space of {G}aussian mixture models.
\newblock \emph{SIAM J. Imaging Sci.}, 13\penalty0 (2):\penalty0 936--970, 2020.

\bibitem[Diao et~al.(2023)Diao, Balasubramanian, Chewi, and Salim]{Diao+23FBGVI}
Michael~Z. Diao, Krishnakumar Balasubramanian, Sinho Chewi, and Adil Salim.
\newblock Forward-backward {G}aussian variational inference via {JKO} in the {B}ures--{W}asserstein space.
\newblock In Andreas Krause, Emma Brunskill, Kyunghyun Cho, Barbara Engelhardt, Sivan Sabato, and Jonathan Scarlett, editors, \emph{Proceedings of the 40th International Conference on Machine Learning}, volume 202 of \emph{Proceedings of Machine Learning Research}, pages 7960--7991. PMLR, 7 2023.

\bibitem[Du et~al.(2019)Du, Zhai, Poczos, and Singh]{Du+19GDNN}
Simon~S. Du, Xiyu Zhai, Barnabas Poczos, and Aarti Singh.
\newblock Gradient descent provably optimizes over-parameterized neural networks.
\newblock In \emph{International Conference on Learning Representations}, 2019.

\bibitem[Ghosh et~al.(2022)Ghosh, Lu, Nowicki, and Zhang]{Gho+22MFVI}
Soumyadip Ghosh, Yingdong Lu, Tomasz Nowicki, and Edith Zhang.
\newblock On representations of mean-field variational inference.
\newblock \emph{arXiv preprint 2210.11385}, 2022.

\bibitem[He et~al.(2015)He, Zhang, Ren, and Sun]{he2015delving}
Kaiming He, Xiangyu Zhang, Shaoqing Ren, and Jian Sun.
\newblock Delving deep into rectifiers: surpassing human-level performance on {I}mage{N}et classification.
\newblock In \emph{Proceedings of the IEEE International Conference on Computer Vision}, pages 1026--1034, 2015.

\bibitem[He et~al.(2016)He, Zhang, Ren, and Sun]{He+16Res}
Kaiming He, Xiangyu Zhang, Shaoqing Ren, and Jian Sun.
\newblock Deep residual learning for image recognition.
\newblock In \emph{2016 IEEE Conference on Computer Vision and Pattern Recognition (CVPR)}, pages 770--778, 2016.

\bibitem[Jacot et~al.(2018)Jacot, Gabriel, and Hongler]{Jac+18NTK}
Arthur Jacot, Franck Gabriel, and Clement Hongler.
\newblock Neural tangent kernel: convergence and generalization in neural networks.
\newblock In S.~Bengio, H.~Wallach, H.~Larochelle, K.~Grauman, N.~Cesa-Bianchi, and R.~Garnett, editors, \emph{Advances in Neural Information Processing Systems}, volume~31. Curran Associates, Inc., 2018.

\bibitem[Jiang et~al.(2023)Jiang, Chewi, and Pooladian]{JiaChePoo23MFVI}
Yiheng Jiang, Sinho Chewi, and Aram-Alexandre Pooladian.
\newblock Algorithms for mean-field variational inference via polyhedral optimization in the {W}asserstein space.
\newblock \emph{arXiv preprint 2312.02849}, 2023.

\bibitem[Jordan et~al.(1998)Jordan, Kinderlehrer, and Otto]{JKO}
Richard Jordan, David Kinderlehrer, and Felix Otto.
\newblock The variational formulation of the {F}okker--{P}lanck equation.
\newblock \emph{SIAM J. Math. Anal.}, 29\penalty0 (1):\penalty0 1--17, 1998.

\bibitem[Lacker(2023)]{Lac23IndepProj}
Daniel Lacker.
\newblock Independent projections of diffusions: gradient flows for variational inference and optimal mean field approximations.
\newblock \emph{arXiv preprint 2309.13332}, 2023.

\bibitem[Lambert et~al.(2022)Lambert, Chewi, Bach, Bonnabel, and Rigollet]{Lam+22GVI}
Marc Lambert, Sinho Chewi, Francis Bach, Silv{\`e}re Bonnabel, and Philippe Rigollet.
\newblock Variational inference via {W}asserstein gradient flows.
\newblock In Alice~H. Oh, Alekh Agarwal, Danielle Belgrave, and Kyunghyun Cho, editors, \emph{Advances in Neural Information Processing Systems}, 2022.

\bibitem[Lambert et~al.(2023)Lambert, Bonnabel, and Bach]{LamBonBac23VarKushner}
Marc Lambert, Silv\`ere Bonnabel, and Francis Bach.
\newblock Variational {G}aussian approximation of the {K}ushner optimal filter.
\newblock In \emph{Geometric science of information. {P}art {I}}, volume 14071 of \emph{Lecture Notes in Comput. Sci.}, pages 395--404. Springer, Cham, 2023.

\bibitem[LeCun and Cortes(2010)]{lecun-mnisthandwrittendigit-2010}
Yann LeCun and Corinna Cortes.
\newblock {MNIST} handwritten digit database.
\newblock 2010.
\newblock URL \url{http://yann.lecun.com/exdb/mnist/}.

\bibitem[Liero et~al.(2016)Liero, Mielke, and Savar\'{e}]{LieMieSav16HK}
Matthias Liero, Alexander Mielke, and Giuseppe Savar\'{e}.
\newblock Optimal transport in competition with reaction: the {H}ellinger--{K}antorovich distance and geodesic curves.
\newblock \emph{SIAM J. Math. Anal.}, 48\penalty0 (4):\penalty0 2869--2911, 2016.

\bibitem[Liero et~al.(2018)Liero, Mielke, and Savar\'{e}]{LieMieSav18HK}
Matthias Liero, Alexander Mielke, and Giuseppe Savar\'{e}.
\newblock Optimal entropy-transport problems and a new {H}ellinger--{K}antorovich distance between positive measures.
\newblock \emph{Invent. Math.}, 211\penalty0 (3):\penalty0 969--1117, 2018.

\bibitem[Lu et~al.(2019)Lu, Lu, and Nolen]{LuLuNol19}
Yulong Lu, Jianfeng Lu, and James Nolen.
\newblock Accelerating {L}angevin sampling with birth-death.
\newblock arXiv, 2019.

\bibitem[Mei et~al.(2018)Mei, Montanari, and Nguyen]{MeiMonNgu18MeanField}
Song Mei, Andrea Montanari, and Phan-Minh Nguyen.
\newblock A mean field view of the landscape of two-layer neural networks.
\newblock \emph{Proc. Natl. Acad. Sci. USA}, 115\penalty0 (33):\penalty0 E7665--E7671, 2018.

\bibitem[Mei et~al.(2019)Mei, Misiakiewicz, and Montanari]{MeiMisMon19MFNN}
Song Mei, Theodor Misiakiewicz, and Andrea Montanari.
\newblock Mean-field theory of two-layers neural networks: dimension-free bounds and kernel limit.
\newblock In Alina Beygelzimer and Daniel Hsu, editors, \emph{Proceedings of the Thirty-Second Conference on Learning Theory}, volume~99 of \emph{Proceedings of Machine Learning Research}, pages 2388--2464. PMLR, 6 2019.

\bibitem[Nitanda et~al.(2022)Nitanda, Wu, and Suzuki]{NitWuSuz22MFLangevin}
Atsushi Nitanda, Denny Wu, and Taiji Suzuki.
\newblock Convex analysis of the mean field {L}angevin dynamics.
\newblock In Gustau Camps-Valls, Francisco J.~R. Ruiz, and Isabel Valera, editors, \emph{Proceedings of the 25th International Conference on Artificial Intelligence and Statistics}, volume 151 of \emph{Proceedings of Machine Learning Research}, pages 9741--9757. PMLR, 3 2022.

\bibitem[Papyan et~al.(2020)Papyan, Han, and Donoho]{PapHanDon20NeuralCollapse}
Vardan Papyan, Xiaoyan Han, and David~L. Donoho.
\newblock Prevalence of neural collapse during the terminal phase of deep learning training.
\newblock \emph{Proc. Natl. Acad. Sci. USA}, 117\penalty0 (40):\penalty0 24652--24663, 2020.

\bibitem[Rotskoff and Vanden-Eijnden(2022)]{RotVan22NNInteracting}
Grant~M. Rotskoff and Eric Vanden-Eijnden.
\newblock Trainability and accuracy of artificial neural networks: an interacting particle system approach.
\newblock \emph{Comm. Pure Appl. Math.}, 75\penalty0 (9):\penalty0 1889--1935, 2022.

\bibitem[Santambrogio(2015)]{San15OT}
Filippo Santambrogio.
\newblock \emph{Optimal transport for applied mathematicians}, volume~87 of \emph{Progress in Nonlinear Differential Equations and their Applications}.
\newblock Birkh\"{a}user/Springer, Cham, 2015.
\newblock Calculus of variations, PDEs, and modeling.

\bibitem[Sirignano and Spiliopoulos(2020)]{SirSpi20MeanFieldCLT}
Justin Sirignano and Konstantinos Spiliopoulos.
\newblock Mean field analysis of neural networks: a central limit theorem.
\newblock \emph{Stochastic Process. Appl.}, 130\penalty0 (3):\penalty0 1820--1852, 2020.

\bibitem[Vaswani et~al.(2017)Vaswani, Shazeer, Parmar, Uszkoreit, Jones, Gomez, Kaiser, and Polosukhin]{Vas+17Attention}
Ashish Vaswani, Noam Shazeer, Niki Parmar, Jakob Uszkoreit, Llion Jones, Aidan~N. Gomez, \L{}ukasz Kaiser, and Illia Polosukhin.
\newblock Attention is all you need.
\newblock In I.~Guyon, U.~Von Luxburg, S.~Bengio, H.~Wallach, R.~Fergus, S.~Vishwanathan, and R.~Garnett, editors, \emph{Advances in Neural Information Processing Systems}, volume~30. Curran Associates, Inc., 2017.

\bibitem[Villani(2003)]{Vil03Topics}
C\'{e}dric Villani.
\newblock \emph{Topics in optimal transportation}, volume~58 of \emph{Graduate Studies in Mathematics}.
\newblock American Mathematical Society, Providence, RI, 2003.

\bibitem[Villani(2009)]{Vil09OT}
C\'{e}dric Villani.
\newblock \emph{Optimal transport}, volume 338 of \emph{Grundlehren der Mathematischen Wissenschaften [Fundamental Principles of Mathematical Sciences]}.
\newblock Springer-Verlag, Berlin, 2009.
\newblock Old and new.

\bibitem[Xiao et~al.(2017)Xiao, Rasul, and Vollgraf]{xiao2017fashion}
Han Xiao, Kashif Rasul, and Roland Vollgraf.
\newblock Fashion-{MNIST}: a novel image dataset for benchmarking machine learning algorithms.
\newblock \emph{arXiv preprint arXiv:1708.07747}, 2017.

\bibitem[Yan et~al.(2023)Yan, Wang, and Rigollet]{yan2023learning}
Yuling Yan, Kaizheng Wang, and Philippe Rigollet.
\newblock Learning {G}aussian mixtures using the {W}asserstein--{F}isher--{R}ao gradient flow.
\newblock \emph{arXiv preprint arXiv:2301.01766}, 2023.

\bibitem[Yao and Yang(2023)]{YaoYan23MFVI}
Rentian Yao and Yun Yang.
\newblock Mean-field variational inference via {W}asserstein gradient flow.
\newblock \emph{arXiv preprint 2207.08074}, 2023.

\end{thebibliography}

\newpage
\appendix

\section{Interpretation as a Wasserstein gradient flow}\label{app:interpretation}

The gradient flows that we consider in this paper are closely related to Wasserstein gradient flows, as we describe in detail here.

\paragraph*{Bures--Wasserstein gradient flows.}
We first consider the case $K = 1$, so that $\rho = \mc N(\mu,\Sigma)$ is simply a Gaussian measure.
The space of non-degenerate Gaussian measures over $\R^D$, which is naturally identified with $\R^d\times \mb S_{++}^d$, can be equipped with the Wasserstein metric and is then known as the \emph{Bures--Wasserstein space}~\citep{BhaJaiLim19BW}. We denote this space by $\BW(\R^D)$. The Riemannian structure of the Wasserstein space endows $\BW(\R^D)$ with a Riemannian metric, called the \emph{Bures--Wasserstein metric}.

Given a functional $\loss$ over the Wasserstein space, we can restrict it to a functional $L : \BW(\R^D) \to \R$ via $L(\mu,\Sigma) \deq \loss(\mc N(\mu,\Sigma))$.
If $\nabla_\mu$, $\nabla_\Sigma$ denote the usual Euclidean gradients of $L$ w.r.t.\ $\mu$ and $\Sigma$ respectively, it is known that the gradient flow of $L$ over $\BW(\R^D)$ is given by
\begin{align}\label{eq:bwgf}
    \boxed{\dot \mu = -\nabla_\mu L(\mu,\Sigma)\,,} \qquad\text{and}\qquad \boxed{\dot\Sigma = -2\,\bigl(\Sigma\,\nabla_\Sigma L(\mu,\Sigma) + \nabla_\Sigma L(\mu,\Sigma)\,\Sigma\bigr)\,.}
\end{align}
See, e.g.,~\citet[Appendix A]{altschuler2021averaging} or~\citet[Appendix B.3]{Lam+22GVI}.
On the other hand, if we parameterize $\Sigma = UU^\top$ where $U \in \R^{D\times D}$ and we follow the Euclidean gradient flow of $(\mu,U) \mapsto L(\mu,UU^\top)$, it is straightforward to see that it
coincides with~\eqref{eq:bwgf}.
This parametrization also has the advantage of maintaining the positive semidefiniteness of $\Sigma$ along the optimization without the need for projections, which is convenient for implementation and can also be used to enforce low-rank factorizations~\citep{BurMon03LowRank}.
In Appendix~\ref{app:derivations}, we derive the Bures--Wasserstein gradient flows for our problems of interest, keeping in mind that they can also be implemented as Euclidean gradient flows via the parametrization $\Sigma = UU^\top$.

\paragraph{Gaussian mixture gradient flows.}
For $K > 1$, the Wasserstein geometry over the space of Gaussian mixtures is not explicit, and we instead follow the geometry defined in~\citet{CheGeoTan19GMM, DelDes20GMM}.
This geometry can be interpreted as the Wasserstein geometry over the curved manifold $(\BW(\R^D), W_2)$, see~\citet[Appendix F]{Lam+22GVI} for details.

\section{Exact derivations}\label{app:derivations}

In this section, we record exact expressions for the Gaussian mixture gradient flows.
These expressions were used to validate the correctness of our GM layer implementations in \texttt{PyTorch} and could be useful for future investigations.

We consider the following two problems.
\begin{itemize}
\item \textbf{Regression.} We are given a dataset ${\{(x_{i},y_{i})\}}_{i \in [n]}$ where each
$x_{i}\in\mathbb{R}^{d}$ and $y_{i}\in\mathbb{R}$. We consider
the squared loss
\[
\ms L(h) = \frac{1}{n} \sum_{i=1}^n {\bigl(y_i - h(x_i)\bigr)}^2\,.
\]
\item \textbf{Multi-class classification.} We are given a dataset ${\{(x_{i},y_{i})\}}_{i\in [n]}$,
where each $x_{i}\in\mathbb{R}^{d}$ and $y_{i}\in\{0,1,\ldots,L\}$ and $L$ denotes the number of classes.
For this problem, we consider the multi-class logistic loss
\[ \ms L(h) = -\frac{1}{n}\sum_{i=1}^{n}\biggl\{ \sum_{\ell=1}^{L} {h(x_i)}_\ell \ind_{y_i=\ell} -\log\Bigl(1+\sum_{\ell=1}^{L}\exp\bigl({h(x_i)}_\ell\bigr)\Bigr)\biggr\}\,,
\]
where $h : \R^d\to\R^{L}$.
\end{itemize}

For regression, we derive exact expressions for the gradient flows for the case of $K=1$ (i.e., the Gaussian mixture is simply a Gaussian) and for the case $K > 1$ for the full parametrization~\eqref{eq:hgmm}. For multi-class classification, we focus on the case $K > 1$ with the diagonal parametrization~\eqref{eq:hcheaper}.

\paragraph{Notation.}
For $\theta = (\omega,\beta)$, we use the shorthand notation $f(x,\theta) \deq \omega \ReLU(\langle \beta,x\rangle)$.
We write $\phi(\cdot \mid \mu,\Sigma)$ for the density of $\mc N(\mu,\Sigma)$.
We also use $\loss$ to denote a loss over the space of probability measures (whereas $\ms L$ denotes a loss over the space of functions).

\subsection{Regression}

\subsubsection{Gaussians}\label{app:regression_gaussian}

We first restrict $\rho$ to be a single Gaussian, namely
\[
\R^D \ni
\theta=\begin{bmatrix}
\omega\\
\beta
\end{bmatrix}\sim\rho = \mathcal{N}(\mu,\Sigma) = \mathcal{N}\Bigl(\begin{bmatrix}
\mu^\omega\\
\mu^{\beta}
\end{bmatrix},\,\begin{bmatrix}
\Sigma^{\omega} & \Sigma^{\omega,\beta}\\
\Sigma^{\beta,\omega} & \Sigma^{\beta}
\end{bmatrix}\Bigr)\,,
\]
where $D = d+1$.
Then we can write the objective function as
\begin{equation}
\min_{\rho=\mathcal{N}(\mu,\Sigma)}\loss(\rho)\coloneqq\frac{1}{n}\sum_{i=1}^{n} {\Bigl(y_i - \int f(x_{i},\theta)\,\rho(\D\theta)\Bigr)^{2}}\,.\label{eq:loss-regression-gaussian}
\end{equation}
The Bures--Wasserstein gradient flow for minimizing $\loss$ is
characterized by the following theorem.

\begin{theorem}\label{thm:regression-gaussian}
    The Bures--Wasserstein
gradient flow $(\rho_{t})_{t\geq0}$ for minimizing $\loss$ in (\ref{eq:loss-regression-gaussian})
is given by $\rho_{t}=\mathcal{N}(\mu_{t},\Sigma_{t})$ that evolves
according to the ODE system
\begin{align*}
\boxed{\begin{aligned}
\dot{\mu_{t}} & =\frac{2}{n}\sum_{i=1}^{n}\bigl(y_{i}-\mathbb{E}_{\rho_{t}} f(x_{i},\theta)\bigr)\,\E_{\rho_{t}}\nabla_{\theta}f(x_{i},\theta)\,,\\
\dot{\Sigma}_{t} & =\frac{2}{n}\sum_{i=1}^{n}\bigl(y_{i}-\mathbb{E}_{\rho_{t}}f(x_{i},\theta)\bigr)\,\E_{\rho_{t}}[\nabla_{\theta}f(x_{i},\theta)\otimes(\theta-\mu_{t})+(\theta-\mu_{t})\otimes\nabla_{\theta}f(x_{i},\theta)]\,.
\end{aligned}}
\end{align*}
\end{theorem}
\begin{proof}
The loss function $\loss$ for the regression problem~\eqref{eq:loss-regression-gaussian} can be written as
\begin{align*}
\loss(\rho)
& =\frac{1}{n}\sum_{i=1}^{n}{\Bigl(y_i - \int f(x_{i},\theta)\,\rho(\D \theta)\Bigr)^{2}} \\
&=\frac{1}{n}\sum_{i=1}^{n}y_{i}^{2}-\frac{2}{n}\sum_{i=1}^{n}y_{i}\int f(x_{i},\theta)\,\rho(\D\theta)+\frac{1}{n}\sum_{i=1}^{n} {\Bigl(\int f(x_{i},\theta)\,\rho(\D\theta)\Bigr)^{2}}\\
 & =\ell_{0}+2\int V(\theta)\,\rho(\D \theta)+\iint U(\theta,\theta')\,\rho(\D\theta)\,\rho(\D\theta')\,,
\end{align*}
where $\ell_{0}=n^{-1}\sum_{i=1}^{n}y_{i}^{2}$ is some constant that
does not depend on $\rho$, and the two functions $V:\mathbb{R}^{D}\to\mathbb{R}$
and $U:\mathbb{R}^{D}\times\mathbb{R}^{D}\to\mathbb{R}$ are given
by
\begin{align}\label{eq:defUV}
    V(\theta)\coloneqq-\frac{1}{n}\sum_{i=1}^{n}y_{i}f(x_{i},\theta)\qquad\text{and}\qquad U(\theta,\theta')\deq \frac{1}{n}\sum_{i=1}^{n}f(x_{i},\theta)\,f(x_{i},\theta')\,.
\end{align}

By~\eqref{eq:bwgf}, it suffices to compute the Euclidean gradients w.r.t.\ the variables $\mu$ and $\Sigma$. We first compute $\nabla_{\mu}\eu L(\rho)$
as follows:
\begin{align}
\nabla_{\mu}\eu L(\rho)
& =2\,\nabla_{\mu}\int V(\theta)\,\phi(\theta\mid \mu,\Sigma)\,\mathrm{d}\theta+\nabla_{\mu}\iint U(\theta,\theta')\,\phi(\theta\mid\mu,\Sigma)\,\phi(\theta'\mid\mu,\Sigma)\,\mathrm{d}\theta\,\mathrm{d}\theta'\nonumber \\
 & \overset{\text{(i)}}{=}2\int V(\theta)\,\nabla_{\mu}\phi(\theta\mid\mu,\Sigma)\,\mathrm{d}\theta+\iint U(\theta,\theta')\,\nabla_{\mu}[\phi(\theta \mid \mu,\Sigma)\,\phi(\theta' \mid \mu,\Sigma)]\,\mathrm{d}\theta\,\mathrm{d}\theta'\nonumber \\
 & \overset{\text{(ii)}}{=}2\int V(\theta)\,\nabla_{\mu}\phi(\theta \mid \mu,\Sigma)\,\mathrm{d}\theta+2\int\Bigl[\int U(\theta,\theta')\,\nabla_{\mu}\phi(\theta \mid \mu,\Sigma)\,\mathrm{d}\theta\Bigr]\,\phi(\theta' \mid \mu,\Sigma)\,\mathrm{d}\theta'\nonumber \\
 & \overset{\text{(iii)}}{=}-2\int V(\theta)\,\nabla_{\theta}\phi(\theta \mid \mu,\Sigma)\,\mathrm{d}\theta-2\int\Bigl[\int U(\theta,\theta')\,\nabla_{\theta}\phi(\theta \mid \mu,\Sigma)\,\mathrm{d}\theta\Bigr]\,\phi(\theta' \mid \mu,\Sigma)\,\mathrm{d}\theta'\nonumber \\
 & \overset{\text{(iv)}}{=}2\int\nabla_{\theta}V(\theta)\,\phi(\theta \mid \mu,\Sigma)\,\mathrm{d}\theta+2\int\Bigl[\int\nabla_{\theta}U(\theta,\theta')\,\phi(\theta \mid \mu,\Sigma)\,\mathrm{d}\theta\Bigr]\,\phi(\theta' \mid \mu,\Sigma)\,\mathrm{d}\theta'\nonumber \\
 &= 2\int\Bigl[\nabla_{\theta}V(\theta)+\int\nabla_{\theta}U(\theta,\theta')\,\phi(\theta' \mid \mu,\Sigma)\,\mathrm{d}\theta'\Bigr]\,\phi(\theta \mid \mu,\Sigma)\,\mathrm{d}\theta\label{eq:regression-grad-mu}
\end{align}
Here, step (i) uses the Leibniz integral rule (since $U$ and $V$
are continuous and $\phi$ is sufficiently smooth); step (ii) follows
from the chain rule and the fact that $U(\theta,\theta')=U(\theta',\theta)$;
step (iii) follows from~\eqref{eq:gaussian-eq-1} in Appendix~\ref{sec:auxiliary};
and step (iv) follows from integration by parts. Following similar
steps, we can compute $\nabla_{\Sigma}\loss(\rho)$ as follows: 
\begin{align*}
\nabla_{\Sigma}\loss(\rho) & =2\int V(\theta)\,\nabla_{\Sigma}\phi(\theta \mid \mu,\Sigma)\,\mathrm{d}\theta+2\int\Bigl[\int U(\theta,\theta')\,\nabla_{\Sigma}\phi(\theta \mid \mu,\Sigma)\,\mathrm{d}\theta\Bigr]\,\phi(\theta' \mid \mu,\Sigma)\,\mathrm{d}\theta'\\
 & \overset{\text{(a)}}{=}\int V(\theta)\,\nabla_{\theta}^{2}\phi(\theta \mid \mu,\Sigma)\,\mathrm{d}\theta+\int\Bigl[\int U(\theta,\theta')\,\nabla_{\theta}^{2}\,\phi(\theta \mid \mu,\Sigma)\,\mathrm{d}\theta\Bigr]\,\phi(\theta' \mid \mu,\Sigma)\,\mathrm{d}\theta'\\
 & \overset{\text{(b)}}{=}-\int\nabla_{\theta}V(\theta)\otimes\nabla_{\theta}\phi(\theta \mid \mu,\Sigma)\,\mathrm{d}\theta \\
 &\qquad{} -\int\Bigl[\int\nabla_{\theta}U(\theta,\theta')\otimes\nabla_{\theta}\phi(\theta \mid \mu,\Sigma)\,\mathrm{d}\theta\Bigr]\,\phi(\theta' \mid \mu,\Sigma)\,\mathrm{d}\theta'\\
 & \overset{\text{(c)}}{=}\int\nabla_{\theta}V(\theta)\otimes(\theta-\mu)\,\phi(\theta;\mu,\Sigma)\,\mathrm{d}\theta\,\Sigma^{-1}\\
 &\qquad{} +\int \Bigl[\int\nabla_{\theta}U(\theta,\theta')\otimes(\theta-\mu)\,\phi(\theta;\mu,\Sigma)\,\mathrm{d}\theta\Bigr]\,\phi(\theta' \mid \mu,\Sigma)\,\mathrm{d}\theta'\,\Sigma^{-1}\\
 & \overset{\text{(b')}}{=}-\int\nabla_{\theta}\phi(\theta \mid \mu,\Sigma)\otimes\nabla_{\theta}V(\theta)\,\mathrm{d}\theta\\
 &\qquad{}-\int\Bigl[\int\nabla_{\theta}\phi(\theta \mid \mu,\Sigma)\otimes\nabla_{\theta}U(\theta,\theta')\,\mathrm{d}\theta\Bigr]\,\phi(\theta' \mid \mu,\Sigma)\,\mathrm{d}\theta'\\
 & \overset{\text{(c')}}{=}\Sigma^{-1}\int(\theta-\mu)\otimes\nabla_{\theta}V(\theta)\,\phi(\theta;\mu,\Sigma)\,\mathrm{d}\theta \\
 &\qquad{} +\Sigma^{-1}\int \Bigl[\int(\theta-\mu)\otimes\nabla_{\theta}U(\theta,\theta')\,\phi(\theta;\mu,\Sigma)\,\mathrm{d}\theta\Bigr]\,\phi(\theta' \mid \mu,\Sigma)\,\mathrm{d}\theta'
\end{align*}
Here, step (a) follows from~\eqref{eq:gaussian-eq-2} in Appendix~\ref{sec:auxiliary};
steps (b) and (b') both follow from applying integration by parts
to the formula following step (a); whereas steps (c) and (c') both
follow from~\eqref{eq:gaussian-eq-0} in Appendix~\ref{sec:auxiliary}.
According to~\eqref{eq:bwgf}, the Bures--Wasserstein
gradient flow for minimizing $\loss$ is given by $\rho_{t}=\mathcal{N}(\mu_{t},\Sigma_{t})$, $t\ge 0$,
where 
\begin{align*}
\dot{\mu}_{t} & =-\nabla_{\mu}\loss(\rho_{t}),\\
\dot{\Sigma}_{t} & =-2\,\nabla_{\Sigma}\loss(\rho_{t})\,\Sigma_{t}-2\,\Sigma_{t}\,\nabla_{\Sigma}\loss(\rho_{t})\,.
\end{align*}
In order to derive a closed-form expression, we compute
\begin{align}
\nabla_{\theta}V(\theta)
& =-\frac{1}{n}\sum_{i=1}^{n}y_{i}\,\nabla_{\theta}f(x_{i},\theta)=-\frac{1}{n}\sum_{i=1}^{n}y_{i}\begin{bmatrix}
\ReLU(\beta^\top x_{i})\\
\omega\ReLU'(\beta^\top x_{i})\,x_{i}
\end{bmatrix}\,,\qquad\text{and}\label{eq:gradV}\\
\nabla_{\theta}U(\theta,\theta') & =\frac{1}{n}\sum_{i=1}^{n}\nabla_{\theta}f(x_{i},\theta)\,f(x_{i},\theta')=\frac{1}{n}\sum_{i=1}^{n}f(x_{i},\theta')\,\begin{bmatrix}
\ReLU(\beta^\top x_{i})\\
\omega\ReLU'(\beta^\top x_{i})\,x_{i}\label{eq:gradU}
\end{bmatrix}\,.
\end{align}
where $\ReLU'(x)=\ind\{x\geq 0\}$ is the derivative of ReLU at any $x\neq 0$.\footnote{In general, the Wasserstein gradient at a measure $\mu$ is an element of $L^2(\mu)$ and therefore is only defined almost everywhere; see~\citet[Chapter 8]{AGS} for details.}
Hence, we arrive at
\begin{align*}
\dot{\mu_{t}} & =-\nabla_{\mu}\loss(\rho_{t})
=-2\int\Bigl[\nabla_{\theta}V(\theta)+\int\nabla_{\theta}U(\theta,\theta')\,\phi(\theta' \mid \mu_{t},\Sigma_{t})\,\mathrm{d}\theta'\Bigr]\,\phi(\theta \mid \mu_{t},\Sigma)\,\mathrm{d}\theta\\
 & =\frac{2}{n}\sum_{i=1}^{n} \bigl(y_{i}-\mathbb{E}_{\rho_{t}}f(x_{i},\theta)\bigr)\,\mathbb{E}_{\rho_{t}}\nabla_{\theta}f(x_{i},\theta)
\end{align*}
and
\begin{align*}
\dot{\Sigma}_{t} & =-2\,\nabla_{\Sigma}\loss(\rho_{t})\,\Sigma_{t}-2\,\Sigma_{t}\,\nabla_{\Sigma}\loss(\rho_{t})\\
 & =-2\int\nabla_{\theta}V(\theta)\otimes(\theta-\mu_{t})\,\phi(\theta\mid \mu_{t},\Sigma_{t})\,\mathrm{d}\theta \\
 &\qquad{} -2\int \Bigl[\int\nabla_{\theta}U(\theta,\theta')\otimes(\theta-\mu_{t})\,\phi(\theta\mid \mu_{t},\Sigma_{t})\,\mathrm{d}\theta\Bigr]\,\phi(\theta' \mid \mu_{t},\Sigma_{t})\,\mathrm{d}\theta'\\
 & \qquad{} -2\int(\theta-\mu_{t})\otimes\nabla_{\theta}V(\theta)\,\phi(\theta\mid \mu_{t},\Sigma_{t})\,\mathrm{d}\theta\\
 &\qquad{} -2\int\Bigl[\int(\theta-\mu_{t})\otimes\nabla_{\theta}U(\theta,\theta')\,\phi(\theta\mid \mu_{t},\Sigma_{t})\,\mathrm{d}\theta\Bigr]\,\phi(\theta' \mid \mu_{t},\Sigma_{t})\,\mathrm{d}\theta'\\
 & =\frac{2}{n}\sum_{i=1}^{n}\bigl(y_{i}-\mathbb{E}_{\rho_{t}} f(x_{i},\theta)\bigr)\,\bigl(\mathbb{E}_{\rho_{t}}[\nabla_{\theta}f(x_{i},\theta)\otimes(\theta-\mu_{t})]+\mathbb{E}_{\rho_{t}}[(\theta-\mu_t)\otimes\nabla_{\theta}f(x_{i},\theta)]\bigr)\,.
\end{align*}
This completes the derivation.
\end{proof}

Although Theorem~\ref{thm:regression-gaussian} derives equations for the Bures--Wasserstein gradient flow, the expressions involve expectations which must also be computed.
We now proceed to show that these expectations can be computed in closed form for ReLU activations, which allows for exact implementation in software.
The following derivations are tedious, but the resulting equations are readily programmed.

We need to compute 
\begin{align*}
\mathbb{E}_{\rho} f(x_{i},\theta) & =\mathbb{E}_{\rho}[\omega\ReLU(\beta^\top x_{i})]\,,\qquad\mathbb{E}_{\rho} \nabla_{\theta}f(x_{i},\theta)=\mathbb{E}_{\rho}\begin{bmatrix}
\ReLU(\beta^\top x_{i})\\
\omega\ReLU'(\beta^\top x_{i})\,x_{i}
\end{bmatrix}\,,\qquad\text{and}\\
\mathbb{E}_{\rho}[\nabla_{\theta}f(x_{i},\theta)\otimes\theta] & =\mathbb{E}_{\rho}\begin{bmatrix}
\omega \ReLU(\beta^\top x_{i}) & \ReLU(\beta^\top x_{i})\,\beta^{\top}\\
\omega^{2}\ReLU'(\beta^\top x_{i})\,x_{i} & \omega\ReLU'(\beta^\top x_{i})\,x_{i}\otimes\beta
\end{bmatrix}\,.
\end{align*}
Basically, we need to compute, for each $1\leq i\leq n$,
\begin{align*}
    A_{i}&=\mathbb{E}_{\rho} \ReLU(\beta^\top x_{i})\,, &B_{i}&=\mathbb{E}_{\rho}[\omega\ReLU'(\beta^\top x_{i})]\,, \\ C_{i}&=\mathbb{E}_{\rho}[\omega \ReLU(\beta^\top x_{i})]\,, &D_{i}&=\mathbb{E}_{\rho}[\omega^{2}\ReLU'(\beta^\top x_{i})]\,,
\end{align*}
and for each $1\leq j\leq d$, 
\[
P_{i,j}=\mathbb{E}_{\rho}[\ReLU(\beta^\top x_i)\,\beta_{j}]\,,\qquad Q_{i,j}=\mathbb{E}_{\rho}[\omega\ReLU'(\beta^\top x_i)\,\beta_{j}]\,.
\]
Then we can express 
\[
\mathbb{E}_{\rho} \nabla_{\theta}f(x_{i},\theta)=\begin{bmatrix}
A_{i}\\
B_{i}x_{i}
\end{bmatrix}\,,\qquad\mathbb{E}_{\rho} f(x_{i},\theta)=C_{i}\,, \]
and
\begin{align*}
\mathbb{E}_{\rho}[\nabla_{\theta}f(x_{i},\theta)\otimes(\theta-\mu)]=\begin{bmatrix}
C_{i} & [P_{i,j}]_{1\leq j\leq d}\\[0.25em]
D_{i}x_{i} & x_{i}\otimes[Q_{i,j}]_{1\leq j\leq d}
\end{bmatrix}-\begin{bmatrix}
A_{i}\\
B_{i}x_{i}
\end{bmatrix}\otimes\mu\,.
\end{align*}
Therefore, the update rule looks like
\[
\mu_{t+1}=\mu_{t}+\eta g_{t}\qquad\text{and}\qquad\Sigma_{t+1}=\Sigma_{t}+\eta G_{t}
\]
where
\[
g_{t}=\frac{2}{n}\sum_{i=1}^{n}(y_{i}-C_{i})\begin{bmatrix}
A_{i}\\
B_{i}x_{i}
\end{bmatrix}
\]
and
\begin{align*}
G_{t} & =\frac{2}{n}\sum_{i=1}^{n}(y_{i}-C_{i})\,\Bigl\{ \begin{bmatrix}
C_{i} & [P_{i,j}]_{1\leq j\leq d}\\[0.25em]
D_{i}x_{i} & x_{i}\otimes[Q_{i,j}]_{1\leq j\leq d}
\end{bmatrix}-\begin{bmatrix}
A_{i}\\
B_{i}x_{i}
\end{bmatrix} \otimes\mu\Bigr\} \\
 & \qquad{} +\frac{2}{n}\sum_{i=1}^{n}(y_{i}-C_{i})\,\Bigl\{ \begin{bmatrix}
C_{i} & D_{i}x_{i}^{\top}\\
[P_{i,j}]_{1\leq j\leq d} & [Q_{i,j}]_{1\leq j\leq d}\otimes x_{i}
\end{bmatrix}-\mu\otimes \begin{bmatrix}
A_{i}\\
B_{i}x_{i}
\end{bmatrix}\Bigr\}\,.
\end{align*}
In addition, the objective function
\[
\loss(\rho_{t})=\frac{1}{n}\sum_{i=1}^{n}(y_{i} - C_i)^{2}\,.
\]

Let us first compute a universal rule. Let $X=\omega$ , $Y=\beta^\top x_i$
and $Z=\beta_{j}$. We have
\begin{align}\label{eq:defXYZ}
\begin{bmatrix}
X\\
Y\\
Z
\end{bmatrix}\sim\mathcal{N}\Biggl(\begin{bmatrix}
\mu_{1}\\
\mu_{2}\\
\mu_{3}
\end{bmatrix},\,\begin{bmatrix}
\sigma_{1}^{2} & \rho_{1,2}\sigma_{1}\sigma_{2} & \rho_{1,3}\sigma_{1}\sigma_{3}\\
\rho_{1,2}\sigma_{1}\sigma_{2} & \sigma_{2}^{2} & \rho_{2,3}\sigma_{2}\sigma_{3}\\
\rho_{1,3}\sigma_{1}\sigma_{3} & \rho_{2,3}\sigma_{2}\sigma_{3} & \sigma_{3}^{2}
\end{bmatrix}\Biggr)\,,
\end{align}
where 
\[
\begin{bmatrix}
\mu_{1}\\
\mu_{2} \\
\mu_{3}
\end{bmatrix}=\begin{bmatrix}
\mu^{\omega}\\[0.25em]
x_{i}^\top \mu^\beta\\[0.25em]
e_{j}^{\top}\mu^{\beta}
\end{bmatrix}\,,\qquad \begin{bmatrix}
\sigma_{1}^{2} & \rho_{1,2}\sigma_{1}\sigma_{2} & \rho_{1,3}\sigma_{1}\sigma_{3}\\
\rho_{1,2}\sigma_{1}\sigma_{2} & \sigma_{2}^{2} & \rho_{2,3}\sigma_{2}\sigma_{3}\\
\rho_{1,3}\sigma_{1}\sigma_{3} & \rho_{2,3}\sigma_{2}\sigma_{3} & \sigma_{3}^{2}
\end{bmatrix}= \begin{bmatrix}
(\sigma^{\omega})^{2} & \Sigma^{\omega,\beta}x_{i} & \Sigma^{\omega,\beta}e_{j}\\[0.25em]
x_{i}^{\top}\Sigma^{\beta,\omega} & x_{i}^{\top}\Sigma^{\beta}x_{i} & x_{i}^{\top}\Sigma^{\beta}e_{j}\\[0.25em]
e_{j}^{\top}\Sigma^{\beta,\omega} & e_{j}^{\top}\Sigma^{\beta}x_{i} & e_{j}^{\top}\Sigma^{\beta}e_{j}
\end{bmatrix}\,.
\]
We know that 
\begin{align*}
\mathsf{cov}(X-\alpha Y,\,Y) & =\rho_{1,2}\sigma_{1}\sigma_{2}-\alpha\sigma_{2}^{2}=0\qquad\text{when}\qquad\alpha=\rho_{1,2}\,\frac{\sigma_{1}}{\sigma_{2}}\,,\qquad\text{and}\\
\mathsf{cov}\bigl(Z-\beta\,(X-\alpha Y)-\gamma Y,\,Y\bigr) & =\rho_{2\,3}\sigma_{2}\sigma_{3}-\gamma\sigma_{2}^{2}=0\qquad\text{when}\qquad\gamma=\rho_{2,3}\,\frac{\sigma_{3}}{\sigma_{2}}\,,\qquad\text{and}
\end{align*}
\begin{align*}
\mathsf{cov}\bigl(Z-\beta\,(X-\alpha Y)-\gamma Y,\,X-\alpha Y\bigr) & =\mathsf{cov}(Z,X-\alpha Y)-\beta\var(X-\alpha Y)=0
\end{align*}
when
\begin{align*}
\beta=\frac{\rho_{1,3}\sigma_{1}\sigma_{3}-\alpha\rho_{2,3}\sigma_{2}\sigma_{3}}{\sigma_{1}^{2}-2\alpha\rho_{1,2}\sigma_{1}\sigma_{2}+\alpha^{2}\sigma_{2}^{2}}\,.
\end{align*}
Therefore $X-\alpha Y$, $Y$, and $Z-\beta\,(X-\alpha Y)-\gamma Y$
are mutually independent. We first compute $A_{i}$ and $B_{i}$.
By direct computation,
\begin{align*}
A_{i} & =\mathbb{E} \max\{ Y,0\} =\mu_{2}+\sigma_{2}\,\mathbb{E}\max\bigl\{ \frac{Y-\mu_{2}}{\sigma_{2}},-\frac{\mu_{2}}{\sigma_{2}}\bigr\} \\
&=\mu_{2}+\sigma_{2}\int_{-\mu_{2}/\sigma_{2}}^{\infty}\frac{x}{\sqrt{2\pi}}\,e^{-x^{2}/2}\,\mathrm{d}x-\mu_{2}\,\Phi\bigl(-\frac{\mu_{2}}{\sigma_{2}}\bigr)\\
 & =\mu_{2}-\sigma_{2}\int_{-\mu_{2}/\sigma_{2}}^{\infty}\frac{1}{\sqrt{2\pi}}\,\mathrm{d}e^{-x^{2}/2}-\mu_{2}\,\Phi\bigl(-\frac{\mu_{2}}{\sigma_{2}}\bigr)
 =\mu_{2}+\sigma_{2}\,\phi\bigl(-\frac{\mu_{2}}{\sigma_{2}}\bigr)-\mu_{2}\,\Phi\bigl(-\frac{\mu_{2}}{\sigma_{2}}\bigr)
\end{align*}
and
\begin{align*}
B_{i} & =\mathbb{E}[X\ind_{Y>0}]=\mathbb{E}[(X-\alpha Y)\ind_{Y>0}]+\mathbb{E}[\alpha Y\ind_{Y>0}]\\
 & =\mathbb{E}[X-\alpha Y]\,\mathbb{P}(Y>0)+\alpha\,\mathbb{E}\max\{ Y,0\} =E_{i}F_{i}+\alpha A_{i}\,,
\end{align*}
where we further define
\begin{align*}
E_{i} & \coloneqq\mathbb{E}[X-\alpha Y]=\mu_{1}-\alpha\mu_{2}\qquad\text{and}\qquad F_{i}\coloneqq\mathbb{P}(Y>0)=1-\Phi\bigl(-\frac{\mu_{2}}{\sigma_{2}}\bigr)\,.
\end{align*}
We then compute
\begin{align*}
C_{i} & =\mathbb{E}[X\max\{ Y,0\}]
=\mathbb{E}[(X-\alpha Y)\max\{ Y,0\} ]+\mathbb{E}[\alpha Y\max\{ Y,0\} ]\\
 & =\mathbb{E}[X-\alpha Y]\,\mathbb{E}[\max\{ Y,0\} ]+\alpha\,\mathbb{E}[Y\max\{ Y,0\} ]
= A_i E_{i}+\alpha G_{i}
\end{align*}
where we define, for $Y\sim\mathcal{N}(\mu_{2},\sigma_{2}^{2})$,
\begin{align*}
G_{i}
& \coloneqq\mathbb{E}[Y\max\{ Y,0\} ]=\mathbb{E}[Y^{2}\ind_{Y>0}]
=\int_{0}^{\infty}\frac{x^{2}}{\sqrt{2\pi}\sigma_{2}}\,e^{-(x-\mu_{2})^{2}/(2\sigma_{2}^{2})}\,\mathrm{d}x\\
 & =\int_{-\mu_{2}/\sigma_{2}}^{\infty}\frac{(\mu_{2}+\sigma_{2}y)^{2}}{\sqrt{2\pi}}\,e^{-y^{2}/2}\,\mathrm{d}y\qquad\text{by change of variable}\qquad y=\frac{x-\mu_{2}}{\sigma_{2}}\\
 & =\mu_{2}^{2}\int_{-\mu_{2}/\sigma_{2}}^{\infty}\frac{1}{\sqrt{2\pi}}\,e^{-y^{2}/2}\,\mathrm{d}y+2\mu_{2}\sigma_{2}\int_{-\mu_{2}/\sigma_{2}}^{\infty}\frac{y}{\sqrt{2\pi}}\,e^{-y^{2}/2}\,\mathrm{d}y+\int_{-\mu_{2}/\sigma_{2}}^{\infty}\frac{\sigma_{2}^{2}y^{2}}{\sqrt{2\pi}}\,e^{-y^{2}/2}\,\mathrm{d}y\\
 & =\mu_{2}^{2}\,[1-\Phi(-\mu_{2}/\sigma_{2})]+2\mu_{2}\sigma_{2}\,\phi(-\mu_{2}/\sigma_{2})\\
 &\qquad{} -\sigma_{2}^{2}\,\frac{y}{\sqrt{2\pi}}\,e^{-y^{2}/2}\Big|_{-\mu_{2}/\sigma_{2}}^{\infty}+\sigma_{2}^{2}\int_{-\mu_{2}/\sigma_{2}}^{\infty}\frac{1}{\sqrt{2\pi}}\,e^{-y^{2}/2}\,\mathrm{d}y\\
 & =(\mu_{2}^{2}+\sigma_{2}^{2})\,[1-\Phi(-\mu_{2}/\sigma_{2})]+\mu_{2}\sigma_{2}\,\phi(-\mu_{2}/\sigma_{2})\,.
\end{align*}
We also need to compute
\begin{align*}
D_{i} & =\mathbb{E}[X^{2}\ind_{Y>0}]=\mathbb{E}[(X-\alpha Y)^{2}\ind_{Y>0}]+\alpha^{2}\,\mathbb{E}[Y^{2}\ind_{Y>0}]+2\alpha\,\mathbb{E}[(X-\alpha Y)Y\ind_{Y>0}]\\
 &= F_i H_{i}+\alpha^{2}G_{i}+2\alpha A_i E_{i}
\end{align*}
where
\[
H_{i}\coloneqq\mathbb{E}[(X-\alpha Y)^{2}]=\sigma_{1}^{2}-2\alpha\rho_{1,2}\sigma_{1}\sigma_{2}+\alpha^{2}\sigma_{2}^{2}+(\mu_{1}-\alpha\mu_{2})^{2}\,.
\]
Then, we compute
\begin{align*}
P_{i,j} & =\mathbb{E}[Z\max\{ Y,0\}]
=\mathbb{E}[(Z-\gamma Y)\max\{ Y,0\} ]+\mathbb{E}[\gamma Y\max\{ Y,0\} ]\\
 & =\mathbb{E}[Z-\gamma Y]\,\mathbb{E}[\max\{ Y,0\} ]+\gamma\,\mathbb{E}[Y\max\{ Y,0\} ]\\
 & = A_i M_{i,j}+\gamma G_{i}\,,
\end{align*}
where we let 
\[
M_{i,j}\coloneqq\mathbb{E}[Z-\gamma Y]=\mu_{3}-\gamma\mu_{2}\,.
\]
Finally, we compute
\begin{align*}
Q_{i,j} & =\mathbb{E}[X\ind_{Y>0}Z]=\mathbb{E}[(X-\alpha Y+\alpha Y)\ind_{Y>0}\,(Z-\beta\,(X-\alpha Y)-\gamma Y+\beta\,(X-\alpha Y)+\gamma Y)]\\
 & =\mathbb{E}[(X-\alpha Y)\ind_{Y>0}\,(Z-\beta\,(X-\alpha Y)-\gamma Y)]+\mathbb{E}[\alpha Y\ind_{Y>0}\,(Z-\beta\,(X-\alpha Y)-\gamma Y)]\\
 & \qquad+\mathbb{E}[(X-\alpha Y)\ind_{Y>0}\,\beta\,(X-\alpha Y)]+\mathbb{E}[\alpha Y\ind_{Y>0}\,\beta\,(X-\alpha Y)]\\
 & \qquad+\mathbb{E}[(X-\alpha Y)\ind_{Y>0}\,\gamma Y]+\mathbb{E}[\alpha Y\ind_{Y>0}\,\gamma Y]\\
 & =E_{i}F_{i}N_{i,j}+\alpha A_{i}N_{i,j}+\beta F_{i} H_i +\alpha\beta A_{i}E_{i}+\gamma A_i E_{i}+\alpha\gamma G_{i}\,,
\end{align*}
where we define
\[
N_{i,j}\coloneqq\mathbb{E}[Z-\beta\,(X-\alpha Y)-\gamma Y]=\mu_{3}-\beta\mu_{1}+\alpha\beta\mu_{2}-\gamma\mu_{2}\,.
\]
We have provided explicit expressions for all of the terms.

\subsubsection{Gaussian mixtures}\label{app:gmm}

Consider a $K$-component Gaussian mixture distribution $\rho_{\nu}$
parameterized by $\nu$:
\[
\rho_{\nu}=\frac{1}{K}\sum_{k=1}^{K}\mathcal{N}(\mu^{(k)},\Sigma^{(k)})\,,\qquad\text{where}\qquad\nu=\frac{1}{K}\sum_{k=1}^{K}\delta_{(\mu^{(k)},\Sigma^{(k)})}\,.
\]
Here $\nu$ is a discrete probability measure over $\mathbb{R}^{d}\times\mb{S}_{++}^{d}$.
We start by deriving the Gaussian mixture gradient flow for a general loss $\loss$.

\begin{theorem}\label{thm:gmm_flow}
    Let $\loss$ be a functional over the Wasserstein space.
    The Gaussian mixture gradient
    flow $(\nu_{t})_{t\geq0}$ for minimizing $\loss$ initialized at a
    distribution $\nu_{0}=K^{-1}\sum_{k=1}^{K}\delta_{(\mu_{0}^{(k)},\Sigma_{0}^{(k)})}$
    with $K$ atoms is given by $\nu_{t}= K^{-1} \sum_{k=1}^{K}\delta_{(\mu_t^{(k)},\Sigma_t^{(k)})}$, $t\ge 0$,
    where for each $k\in[K]$, the dynamics of $(\mu_{t}^{(k)})_{t\geq0}$
    and $(\Sigma_{t}^{(k)})_{t\geq0}$ are governed by the ODE system
    \begin{align*}
    \boxed{\begin{aligned}
    \dot{\mu}_{t}^{(k)} &= - \E_{\mc N(\mu_t^{(k)},\Sigma_t^{(k)})}\nabla \delta \loss (\rho_{\nu_t})\,, \\
    \dot\Sigma_t^{(k)}
    &= -\Sigma_t^{(k)}\,\E_{\mc N(\mu_t^{(k)},\Sigma_t^{(k)})}\nabla^2 \delta \loss(\rho_{\nu_t}) - \E_{\mc N(\mu_t^{(k)},\Sigma_t^{(k)})}\nabla^2 \delta \loss(\rho_{\nu_t}) \,\Sigma_t^{(k)}\,.
    \end{aligned}}
    \end{align*}
    Here, $\delta\loss = \delta_\rho \loss$ refers to the first variation of $\loss$~\citep[see][Chapter 7]{San15OT}.
\end{theorem}
\begin{proof}
We refer to~\citet[Appendix F]{Lam+22GVI} for background.
We calculate the first variation of $\nu \mapsto \loss(\rho_\nu)$ in terms of the first variation of $\loss$: note that for any $\delta>0$
and any measure $\mathcal{X}$ satisfying $\int_{\mathsf{BW}(\mathbb{R}^{d})}\mathrm{d}\mathcal{X}=0$ and such that $\nu+\delta\mc X$ for sufficiently small $\delta > 0$, if $\xi = (\mu,\Sigma)$ and $p_\xi$ denotes $\mc N(\mu,\Sigma)$,
we have
\begin{align*}
    \loss(\rho_{\nu+\delta\mc X}) - \loss(\rho_\nu)
    &= \int \delta\loss(\rho_\nu)\,\D (\rho_{\nu+\delta\mc X}-\rho_\nu) + o(\delta)
    = \int \Bigl[\int \delta \loss(\rho_\nu)\, \D p_\xi\Bigr] \,\delta \mc X(\D \xi) + o(\delta)\,,
\end{align*}
which, by the definition of the first variation, shows that
\[
\delta_\nu\loss(\rho_\nu):\xi\mapsto \int \delta\loss(\rho_\nu)\,\D p_\xi\,.
\]
Since the Gaussian mixture gradient flow is, by definition, a Wasserstein gradient flow over the Bures--Wasserstein space, the particle interpretation of Wasserstein gradient flows shows that each $(\mu_t^{(k)},\Sigma_t^{(k)})$ evolves by the Bures--Wasserstein gradient of $\delta_\nu \loss(\rho_{\nu_t})$.
It follows from~\citet[Appendix C]{Lam+22GVI} that the Gaussian mixture flow takes the claimed form.

Alternatively, we can derive the gradient flow more explicitly.
Noting that $\xi=(\mu,\Sigma)$, we have
\begin{align*}
\nabla_{\mu}\delta_\nu\loss(\rho_\nu)(\xi)
&= \int \delta \loss(\rho_\nu) \, \nabla_\mu \phi(\theta \mid \mu,\Sigma) \, \D \theta
= -\int \delta \loss(\rho_\nu) \, \nabla_\theta \phi(\theta \mid \mu,\Sigma) \, \D \theta \\
&= \int \nabla_\theta \delta \loss(\rho_\nu) \, \phi(\theta \mid \mu,\Sigma)\,\D \theta
\end{align*}
and
\begin{align*}
\nabla_{\Sigma}\delta_\nu\loss(\rho_\nu)(\xi)
&= \int \delta \loss(\rho_\nu) \, \nabla_\Sigma \phi(\theta \mid \mu,\Sigma) \, \D \theta
= \frac{1}{2}\int \delta \loss(\rho_\nu) \, \nabla^2_\theta \phi(\theta \mid \mu,\Sigma) \, \D \theta \\
&= \frac{1}{2} \int \nabla_\theta^2 \delta \loss(\rho_\nu) \, \phi(\theta \mid \mu,\Sigma)\,\D \theta\,,
\end{align*}
where we used the expressions in Appendix~\ref{sec:auxiliary}.
Recalling~\eqref{eq:bwgf}, it completes the derivation.
\end{proof}

With Theorem~\ref{thm:gmm_flow} in hand, we can now prove Theorem~\ref{thm:main_gmm}.

\begin{proof}[Proof of Theorem~\ref{thm:main_gmm}]
Consider the loss function 
\[
\ell(\bs\mu,\bs C) \equiv \ell(\mu_1,\ldots,\mu_K,C_1,\ldots,C_K) \coloneqq \mathcal{L}(\rho_\nu)
\]
where
\[
\nu = \frac{1}{K} \sum_{k=1}^{K}\delta_{(\mu_k,\Sigma_k)}\quad \text{with}\quad \Sigma_k = C_k C_k^\top\,.
\]
Note that for any $\Delta \in \mathbb{R}^d$ and any $\varepsilon>0$, we have
\[
    \lim_{\varepsilon\to 0} \frac{\ell(\mu_1 +\varepsilon\Delta, \mu_2,\ldots,\mu_K,C_1,\ldots,C_K) - \ell(\mu_1,\ldots,\mu_K,C_1,\ldots,C_K)}{\varepsilon} = \langle \nabla_{\mu_1} \ell(\bs\mu,\bs C), \Delta \rangle\,.
\]
The left-hand side of the above equation can also be expressed as
\[
\lim_{\varepsilon\to 0} \frac{\mathcal{L}(\rho_{\nu_\varepsilon}) - \mathcal{L}(\rho_{\nu})}{\varepsilon} \qquad \text{where} \qquad \nu_\varepsilon \coloneqq \frac{1}{K}\, \delta_{(\mu_1+\varepsilon \Delta,\Sigma_1)}+ \frac{1}{K} \sum_{k=2}^{K}\delta_{(\mu_k,\Sigma_k)}\,.
\]
By the definition of first variation, we know that
\begin{align*}
\mathcal{L}(\rho_{\nu_\varepsilon}) - \mathcal{L}(\rho_{\nu}) 
&= \int \delta\loss(\rho_\nu)\,\D (\rho_{\nu_\varepsilon}-\rho_\nu) + o(\varepsilon) \\
&=\int \Bigl[\int \delta \loss(\rho_\nu)\, \D p_\xi\Bigr] \,\varepsilon \mc X(\D \xi) + o(\varepsilon) \qquad \text{where}\qquad 
\mathcal{X}= \nu_\varepsilon - \nu \\
&= \frac{1}{K}\, \Big[ \int \delta \loss(\rho_\nu)\, \D p_{(\mu_1+\varepsilon\Delta,\Sigma_1)} - \int \delta \loss(\rho_\nu)\, \D p_{(\mu_1,\Sigma_1)} \Big] + o(\varepsilon)\,.
\end{align*}
Taking the above three relations collectively yields
\[
\langle \nabla_{\mu_1} \ell(\bs\mu,\bs C), \Delta \rangle = \frac{1}{K}\, \Big\langle\nabla_{\mu_1} \int \delta \loss(\rho_\nu)\, \D p_{(\mu_1,\Sigma_1)}, \Delta \Big\rangle\,.
\]
Since the above equation holds for any $\Delta\in\mathbb{R}^d$, we know that
\[
\nabla_{\mu_1} \ell(\bs\mu,\bs C) = \frac{1}{K}\, \nabla_{\mu_1} \int \delta \loss(\rho_\nu)\, \D p_{(\mu_1,\Sigma_1)} = \frac{1}{K}\, \E_{\mc N(\mu_1,\Sigma_1)}\nabla \delta \loss (\rho_{\nu_t})\,,
\]
where we used the expressions in Appendix~\ref{sec:auxiliary} and integration by parts. Therefore the Euclidean gradient flow w.r.t.~$\mu^{(k)}$ is given by
\begin{equation} \label{eq:Euclidean-GF-mu}
\dot{\mu}^{(k)}_{t} = -\frac{1}{K}\, \E_{\mc N(\mu^{(k)}_t,\Sigma^{(k)}_t)}\nabla \delta \loss (\rho_{\nu_t})\,.
\end{equation}
Similarly, we can show that
\[
\nabla_{\Sigma_1} \ell(\bs\mu,\bs C) = \frac{1}{K}\, \nabla_{\Sigma_1} \int \delta \loss(\rho_\nu)\, \D p_{(\mu_1,\Sigma_1)} = \frac{1}{2K}\, \E_{\mc N(\mu_1,\Sigma_1)}\nabla^2 \delta \loss (\rho_{\nu_t})\,,
\]
which then leads to
\[
\nabla_{C_1} \ell(\bs\mu,\bs C) = 2\,\nabla_{\Sigma_1} \ell(\bs\mu,\bs C)\, C_1 = \frac{1}{K}\, \E_{\mc N(\mu_1,\Sigma_1)}\nabla^2 \delta \loss (\rho_{\nu_t}) \,C_1\,.
\]
Hence the Euclidean gradient flow w.r.t.~$C^{(k)}$ is given by
\[
\dot{C}^{(k)}_{t} = -\frac{1}{K}\, \E_{\mc N(\mu^{(k)}_t,\Sigma^{(k)}_t)} \nabla^2 \delta \loss (\rho_{\nu_t}) \,C^{(k)}_t\,,
\]
therefore $\Sigma^{(k)}_t = C^{(k)}_t C^{(k)\top}_t$ satisfies
\begin{align} 
\dot{\Sigma}^{(k)}_{t} &= C^{(k)}_t \dot{C}^{(k)\top}_{t} + \dot{C}^{(k)}_{t} C^{(k)\top}_t \nonumber\\
&= -\frac{1}{K}\, \big[\E_{\mc N(\mu^{(k)}_t,\Sigma^{(k)}_t)} \nabla^2 \delta \loss (\rho_{\nu_t}) \big] \,\Sigma^{(k)}_t -\frac{1}{K} \,\Sigma^{(k)}_t\,\big[\E_{\mc N(\mu^{(k)}_t,\Sigma^{(k)}_t)} \nabla^2 \delta \loss (\rho_{\nu_t}) \big]\,.\label{eq:Euclidean-GF-Sigma}
\end{align}
By comparing \eqref{eq:Euclidean-GF-mu} and \eqref{eq:Euclidean-GF-Sigma} with the ODE system in Theorem~\ref{thm:gmm_flow}, we can see that they are equivalent up to a scaling factor of $K$.
\end{proof}

Then, we specialize to the objective function 
\[
\min_{\nu}\loss(\rho_{\nu})\coloneqq\frac{1}{n}\sum_{i=1}^{n} {\Bigl(y_i - \int f(x_{i},\theta)\,\rho_{\nu}(\D \theta)\Bigr)^{2}}\,.
\]

\begin{theorem} \label{thm:regression-gaussian-gmm}
The Gaussian mixture gradient
flow $(\nu_{t})_{t\geq0}$ for minimizing $\loss$ initialized at a
distribution $\nu_{0}=K^{-1}\sum_{k=1}^{K}\delta_{(\mu_{0}^{(k)},\Sigma_{0}^{(k)})}$
with $K$ atoms is given by $\nu_{t}= K^{-1} \sum_{k=1}^{K}\delta_{(\mu_t^{(k)},\Sigma_t^{(k)})}$, $t\ge 0$,
where for each $k\in[K]$, the dynamics of $(\mu_{t}^{(k)})_{t\geq0}$
and $(\Sigma_{t}^{(k)})_{t\geq0}$ are governed by the ODE system
\begin{align*}
\boxed{\begin{aligned}
\dot{\mu}_{t}^{(k)} & =\frac{2}{n}\sum_{i=1}^{n}\bigl(y_{i}-\mathbb{E}_{\rho_{\nu_{t}}} f(x_{i},\theta)\bigr)\,\mathbb{E}_{\mathcal{N}(\mu_{t}^{(k)},\Sigma_{t}^{(k)})} \nabla_{\theta}f(x_{i},\theta)\,,\\
\dot{\Sigma}_{t}^{(k)} & =\frac{2}{n}\sum_{i=1}^{n}\bigl(y_{i}-\mathbb{E}_{\rho_{\nu_{t}}} f(x_{i},\theta)\bigr)\\
&\qquad\qquad{}\times \mathbb{E}_{\mathcal{N}(\mu_{t}^{(k)},\Sigma_{t}^{(k)})}[\nabla_{\theta}f(x_{i},\theta)\otimes(\theta-\mu_{t}^{(k)})+(\theta-\mu_{t}^{(k)})\otimes\nabla_{\theta}f(x_{i},\theta)]\,.
\end{aligned}}
\end{align*}
\end{theorem}
\begin{proof}
    The first variation is given by
    \begin{align*}
        \delta \loss(\rho) : \theta \mapsto V(\theta) + \int U(\theta,\theta')\,\rho(\D \theta')\,,
    \end{align*}
    where $U$ and $V$ are as in~\eqref{eq:defUV}.
    Then, we can apply Theorem~\ref{thm:gmm_flow}, which is seen to yield the desired equations after substituting in the definitions of $U$ and $V$ and performing integration by parts.
\end{proof}

As before, we write out more explicit expressions for the gradient flow.
Thankfully, we can reuse our previous calculations.
Compared to Appendix~\ref{app:regression_gaussian}, we only need to compute 
\[
C_{i}'\coloneqq\mathbb{E}_{\rho_\nu}f(x_{i},\theta)=\mathbb{E}_{\rho_{\nu}}[\omega\ReLU(\beta^\top x_i)]
=\frac{1}{K}\sum_{k=1}^{K}\mathbb{E}_{\mathcal{N}(\mu^{(k)},\Sigma^{(k)})}[\omega\ReLU(\beta^\top x_i)]=\frac{1}{K}\sum_{k=1}^{K}C_{i}^{(k)}\,.
\]
Then, the update becomes
\[
\mu_{t+1}^{(k)}=\mu_{t}^{(k)}+\eta g_{t}^{(k)}\qquad\text{and}\qquad\Sigma_{t+1}^{(k)}=\Sigma_{t}^{(k)}+\eta G_{t}^{(k)}
\]
where
\[
g_{t}^{(k)}=\frac{2}{n}\sum_{i=1}^{n}(y_{i}-C_{i}')\begin{bmatrix}
A_{i}\\
B_{i}x_{i}
\end{bmatrix}
\]
and
\begin{align*}
G_{t}^{(k)}
& =\frac{2}{n}\sum_{i=1}^{n}(y_{i}-C_{i}')\,\Bigl\{ \begin{bmatrix}
C_{i} & [P_{i,j}]_{1\leq j\leq d}\\[0.25em]
D_{i}x_{i} & x_{i}\otimes[Q_{i,j}]_{1\leq j\leq d}
\end{bmatrix}-\begin{bmatrix}
A_{i}\\
B_{i}x_{i}
\end{bmatrix}\otimes\mu\Bigr\} \\
 & \qquad+\frac{2}{n}\sum_{i=1}^{n}(y_{i}-C_{i}')\,\Bigl\{ \begin{bmatrix}
C_{i} & D_{i}x_{i}^{\top}\\
[P_{i,j}]_{1\leq j\leq d} & [Q_{i,j}]_{1\leq j\leq d}\otimes x_{i}
\end{bmatrix}-\mu\otimes\begin{bmatrix}
A_{i}\\
B_{i}x_{i}
\end{bmatrix}\Bigr\}\, .
\end{align*}

\subsection{Multi-class classification}

While our implementation relies on PyTorch's Automatic Differentiation engine, we present here exact computations of gradients for multi-class classification. Their complexity indicates that it is largely preferable to employ automatic differentiation and that a study of first-order optimality conditions appears challenging.

For this setting, we focus on the parametrization as described in Subsection~\ref{subsec:reducing_param}, that is, $\beta \sim \mc N(\mu^\beta, \Sigma)$ and $\E[\omega \mid \beta] = U\beta + v$, where we apply the Euclidean gradient flow to the parameters $\mu^\beta$, $U$, $v$, and the square root of $\Sigma$.
To simplify the notation, we write $\mu = \mu^\beta$.
(The sparse parametrization in Subsection~\ref{subsec:reducing_param} corresponds to further restricting $\Sigma$ to be diagonal.)

We first compute the gradients w.r.t.\ $U$ and $v$. Write
\[
U=\begin{bmatrix}
u_{1}^{\top}\\
\vdots\\
u_{L}^{\top}
\end{bmatrix}\qquad\text{and}\qquad v=\begin{bmatrix}
v_{1}\\
\vdots\\
v_{L}
\end{bmatrix}\,.
\]
The following expectations are understood to be taken over the Gaussian mixture. We have
\begin{align*}
    \mathbb{E}\nabla_{\beta}\bigl(\ReLU(\beta^\top x )\,u^{\top}\beta\bigr)
    &=\mathbb{E}[\nabla_{\beta} \ReLU(\beta^\top x )\,u^{\top}\beta+\ReLU(\beta^\top x )\,u] \\
    &=\mathbb{E}[u^{\top}\beta\,\ReLU'(\beta^\top x)\,\beta]+\mathbb{E}[\ReLU(\beta^\top x )]\,u\,.
\end{align*}
The loss function is
\begin{align*}
\loss(\rho) & =-\frac{1}{n}\sum_{i=1}^{n}\Bigl\{ \sum_{\ell=1}^{L}\ind_{y_i=\ell}\int f_{\ell}(x_{i},\theta)\,\rho(\D\theta)-\log\Bigl[1+\sum_{\ell=1}^{L}\exp\int f_{\ell}(x_{i},\theta)\,\rho(\D\theta)\Bigr]\Bigr\} 
\end{align*}
where
\[
\int f_{\ell}(x_{i},\theta)\,\rho(\D\theta)
=\mathbb{E}[\ReLU(\beta^\top x ) \,u_{\ell}^{\top}\beta]+v_{\ell}\,\mathbb{E} \ReLU(\beta^\top x )\,.
\]
We can compute
\begin{align*}
\nabla_{u_{\ell}}\loss(\rho) & =-\frac{1}{n}\sum_{i=1}^{n}\Bigl\{ \ind_{y_{i}=\ell}-\frac{\exp\int f_{\ell}(x_{i},\theta)\,\rho(\D\theta)}{1+\sum_{\ell'=1}^{L}\exp\int f_{\ell'}(x_{i},\theta)\,\rho(\D\theta)}\Bigr\}\, \mathbb{E}[\ReLU(\beta^\top x )\,\beta]\,,\\
\nabla_{v_{\ell}}\loss(\rho) & =-\frac{1}{n}\sum_{i=1}^{n}\Bigl\{ \ind_{y_{i}=\ell}-\frac{\exp\int f_{\ell}(x_{i},\theta)\,\rho(\D\theta)}{1+\sum_{\ell'=1}^{L}\exp\int f_{\ell'}(x_{i},\theta)\,\rho(\D\theta)}\Bigr\}\, \mathbb{E}\ReLU(\beta^\top x)\,.
\end{align*}
Next, we can compute
\begin{align*}
\nabla_{\mu}\loss(\rho)
& =-\frac{1}{n}\sum_{i=1}^{n}\Bigl\{ \sum_{\ell=1}^{L}\ind_{y_{i}=\ell}\nabla_{\mu}\int f_{\ell}(x_{i},\theta)\,\rho(\D\theta)\\
&\qquad\qquad\qquad\qquad\qquad{} -\frac{\sum_{\ell=1}^{L}\exp(\int f_{\ell}(x_{i},\theta)\,\rho(\D\theta))\,\nabla_{\mu}\int f_{\ell}(x_{i},\theta)\,\rho(\D\theta)}{1+\sum_{\ell=1}^{L}\exp\int f_{\ell}(x_{i},\theta)\,\rho(\D\theta)}\Bigr\} \\
 & =-\frac{1}{n}\sum_{i=1}^{n}\sum_{\ell=1}^{L}\Bigl\{ \ind_{y_{i}=\ell}-\frac{\exp\int f_{\ell}(x_{i},\theta)\,\rho(\D\theta)}{1+\sum_{\ell=1}^{L}\exp\int f_{\ell}(x_{i},\theta)\,\rho(\D\theta)}\Bigr\}\, \nabla_{\mu}\int f_{\ell}(x_{i},\theta)\,\rho(\D\theta)
\end{align*}
and similarly,
\[
\nabla_{\Sigma}\loss(\rho)=-\frac{1}{n}\sum_{i=1}^{n}\sum_{\ell=1}^{L}\Bigl\{ \ind_{y_{i}=\ell}-\frac{\exp\int f_{\ell}(x_{i},\theta)\,\rho(\D\theta)}{1+\sum_{\ell=1}^{L}\exp\int f_{\ell}(x_{i},\theta)\,\rho(\D\theta)}\Bigr\}\, \nabla_{\Sigma}\int f_{\ell}(x_{i},\theta)\,\rho(\D\theta)\,.
\]
We have
\begin{align*}
\nabla_{\mu}\int f_{\ell}(x_{i},\theta)\,\rho(\D\theta) & =\nabla_{\mu}\mathbb{E}[\ReLU(\beta^\top x )\,u_{\ell}^{\top}\beta]+\nabla_{\mu}v_{\ell}\,\mathbb{E}\ReLU(\beta^\top x )\\
 & =\mathbb{E}\nabla_{\beta}\bigl(\ReLU(\beta^\top x )\,u_{\ell}^{\top}\beta\bigr)+v_{\ell}\,\mathbb{E} \nabla_{\beta}\ReLU(\beta^\top x )\\
 & =\mathbb{E}[\ReLU'(\beta^\top x )\,u_{\ell}^{\top}\beta]\,x+\mathbb{E}[\ReLU(\beta^\top x )]\,u_{\ell}+v_{\ell}\,\mathbb{E}[\ReLU'(\beta^\top x )]\,x\,.
\end{align*}
We also have
\begin{align*}
\nabla_{\Sigma}\int f_{\ell}(x_{i},\theta)\,\rho(\D\theta) & =\nabla_{\Sigma}\mathbb{E}[\ReLU(\beta^\top x )\,u_{\ell}^{\top}\beta+v_{\ell}\,\ReLU(\beta^\top x )]\,,
\end{align*}
where
\begin{align*}
\nabla_{\Sigma}\mathbb{E}[\ReLU(\beta^\top x )\,u_{\ell}^{\top}\beta]
& =\frac{1}{2}\,\mathbb{E}\bigl[\nabla_{\beta}[\ReLU(\beta^\top x )\,u_{\ell}^{\top}\beta]\otimes(\beta-\mu)\bigr]\,\Sigma^{-1} \\
&=\frac{1}{2}\,\Sigma^{-1}\,\mathbb{E}\bigl[(\beta-\mu)\otimes\nabla_{\beta}[\ReLU(\beta^\top x )\,u_{\ell}^{\top}\beta]\bigr]\\
 & =\frac{1}{2}\,\mathbb{E}\bigl[\bigl(\ReLU'(\beta^\top x )\,u_{\ell}^{\top}\beta\,x+\ReLU(\beta^\top x )\,u_{\ell}\bigr)\otimes(\beta-\mu)\bigr]\,\Sigma^{-1}\\
 &=\frac{1}{2}\,\Sigma^{-1}\,\mathbb{E}\bigl[(\beta-\mu)\otimes\bigl(\ReLU'(\beta^\top x )\,u_{\ell}^{\top}\beta\,x+\ReLU(\beta^\top x )\,u_{\ell}\bigr)\bigr]
\end{align*}
and
\begin{align*}
\nabla_{\Sigma}\mathbb{E}\ReLU(\beta^\top x )
& =\frac{1}{2}\,\mathbb{E}[\nabla_{\beta} \ReLU(\beta^\top x ) \otimes(\beta-\mu)]\,\Sigma^{-1}=\frac{1}{2}\,\Sigma^{-1}\,\mathbb{E}[(\beta-\mu)\otimes\nabla_{\beta}\ReLU(\beta^\top x )]\\
 & =\frac{1}{2}\,\mathbb{E}[\ReLU'(\beta^\top x )\,x\otimes(\beta-\mu)]\,\Sigma^{-1}=\frac{1}{2}\,\Sigma^{-1}\,\mathbb{E}[(\beta-\mu)\otimes\ReLU'(\beta^\top x )\,x]\,.
\end{align*}
Hence, we have
\begin{align*}
\dot{\mu} & =-\nabla_{\mu}\loss(\rho)
=\frac{1}{n}\sum_{i=1}^{n}\sum_{\ell=1}^{L}\Bigl[\ind_{y_{i}=\ell}-\frac{\exp\int f_{\ell}(x_{i},\theta)\,\rho(\D\theta)}{1+\sum_{\ell=1}^{L}\exp\int f_{\ell}(x_{i},\theta)\,\rho(\D\theta)}\Bigr]\,\nabla_{\mu}\int f_{\ell}(x_{i},\theta)\,\rho(\D\theta)\\
 & =\frac{1}{n}\sum_{i=1}^{n}\sum_{\ell=1}^{L}\Bigl[\ind_{y_{i}=\ell}-\frac{\exp\int f_{\ell}(x_{i},\theta)\,\rho(\D\theta)}{1+\sum_{\ell=1}^{L-1}\exp\int f_{\ell}(x_{i},\theta)\,\rho(\D\theta)}\Bigr] \\
 &\qquad\qquad\qquad{}\times \bigl[\mathbb{E}[\ReLU'(\beta^\top x ) \,u_{\ell}^{\top}\beta]\,x+\mathbb{E}[\ReLU(\beta^\top x )]\,u_{\ell}+v_{\ell}\,\mathbb{E}[\ReLU'(\beta^\top x )]\,x\bigr]\\
 & =\frac{1}{n}\sum_{i=1}^{n}\sum_{\ell=1}^{L-1}\Bigl[\ind_{y_{i}=\ell}-\frac{\exp\int f_{\ell}(x_{i},\theta)\,\rho(\D\theta)}{1+\sum_{\ell=1}^{L-1}\exp\int f_{\ell}(x_{i},\theta)\,\rho(\D\theta)}\Bigr] \\
 &\qquad\qquad\qquad{}\times \bigl[u_{\ell}^{\top}\mathbb{E}[\ReLU'(\beta^\top x_i)\,\beta]\,x_{i}+\mathbb{E}[\ReLU(\beta^\top x_i)]\,u_{\ell}+v_{\ell}\,\mathbb{E}[\ReLU'(\beta^\top x )]\,x\bigr]
\end{align*}
and
\begin{align*}
\dot{\Sigma}
& = -2\,\nabla_{\Sigma}\loss(\rho)\,\Sigma-2\,\Sigma\,\nabla_{\Sigma}\loss(\rho)\\
 & =\frac{1}{n}\sum_{i=1}^{n}\sum_{\ell=1}^{L-1}\Bigl[\ind_{y_{i}=\ell}-\frac{\exp\int f_{\ell}(x_{i},\theta)\,\rho(\D\theta)}{1+\sum_{\ell=1}^{L-1}\exp\int f_{\ell}(x_{i},\theta)\,\rho(\D\theta)}\Bigr]\\
 &\qquad\quad {}\times \bigl\{\mathbb{E}\bigl[\bigl(\ReLU'(\beta^\top x_i)\,u_{\ell}^{\top}\beta\,x_{i}+\ReLU(\beta^\top x_i)\,u_{\ell}+\ReLU'(\beta^\top x_i)\,x_{i}\bigr)\otimes(\beta-\mu)\bigr] \\
 &\qquad\qquad\quad+\mathbb{E}\bigl[\bigl(\ReLU'(\beta^\top x_i)\,u_{\ell}^{\top}\beta\,x_{i}+\ReLU(\beta^\top x_i)\,u_{\ell}+\ReLU'(\beta^\top x_i)\,x_{i}\bigr)\otimes(\beta-\mu)\bigr]^\top\bigr\}\,.
\end{align*}
Let
\[
A_{i}=\mathbb{E}\ReLU(\beta^\top x_i)\,,\qquad F_{i}=\mathbb{E}\ReLU'(\beta^\top x_i)\,,\qquad R_{i,j}=\mathbb{E}[\ReLU'(\beta^\top x_i)\,\beta_{j}]\,,
\]
and
\[
P_{i,j}=\mathbb{E}[\ReLU(\beta^\top x_i)\,\beta_{j}]\,,\qquad Q_{i,j,\ell}=\mathbb{E}[\ReLU'(\beta^\top x_i)\,u_{\ell}^{\top}\beta\,\beta_{j}]\,,\qquad S_{i,\ell} = \mathbb{E}[\omega_{\ell}\ReLU(\beta^\top x_i)]\,.
\]
We have
\begin{align*}
\dot{\mu}
& =\frac{1}{n}\sum_{i=1}^{n}\sum_{\ell=1}^{L-1}\Bigl[\ind_{y_{i}=\ell}-\frac{\exp S_{i,\ell}}{1+\sum_{\ell'=1}^{L-1}\exp S_{i,\ell'}}\Bigr]\,[u_{\ell}^{\top}P_{i} \,x_{i}+A_{i}u_{\ell}+v_{\ell}\,F_{i}x_{i}]
\end{align*}
and
\begin{align*}
\dot{\Sigma}
& =\frac{1}{n}\sum_{i=1}^{n}\sum_{\ell=1}^{L-1}\Bigl[\ind_{y_{i}=\ell}-\frac{\exp S_{i,\ell}}{1+\sum_{\ell'=1}^{L-1}\exp S_{i,\ell'}}\Bigr] \\
&\qquad{}\times \big[(x_{i}\otimes Q_{i,\ell}+u_{\ell}\otimes P_{i}+x_{i}\otimes R_{i}- u_{\ell}^{\top}R_{i}\,x_{i}\otimes\mu-A_{i}u_{\ell}\otimes\mu-F_{i}x_{i}\otimes\mu)+ \\
& \qquad\qquad + (x_{i}\otimes Q_{i,\ell}+u_{\ell}\otimes P_{i}+x_{i}\otimes R_{i}- u_{\ell}^{\top}R_{i}\,x_{i}\otimes\mu-A_{i}u_{\ell}\otimes\mu-F_{i}x_{i}\otimes\mu)^\top \big]\,.
\end{align*}
We also have
\[
\dot{U}
=-\begin{bmatrix}
\nabla_{u_{1}}\loss(\rho)^{\top}\\
\vdots\\
\nabla_{u_{L-1}}\loss(\rho)^{\top}
\end{bmatrix}
=\frac{1}{n}\sum_{i=1}^{n}\begin{bmatrix}
\ind_{y_{i}=1}-\frac{\exp S_{i,1} }{1+\sum_{\ell'=1}^{L-1}\exp S_{i,\ell'}}\\
\vdots\\
\ind_{y_{i}=L-1}-\frac{\exp S_{i,L-1}}{1+\sum_{\ell'=1}^{L-1}\exp S_{i,\ell'}}
\end{bmatrix}\otimes P_{i}
\]
and
\[
\dot{v}
=-\begin{bmatrix}
\nabla_{v_{1}}\loss(\rho)^{\top}\\
\vdots\\
\nabla_{v_{L-1}}\loss(\rho)^{\top}
\end{bmatrix}=\frac{1}{n}\sum_{i=1}^{n}\begin{bmatrix}
\ind_{y_{i}=1}-\frac{\exp S_{i,1}}{1+\sum_{\ell'=1}^{L-1}\exp S_{i,\ell'}}\\
\vdots\\
\ind_{y_{i}=L-1}-\frac{\exp S_{i,L-1}}{1+\sum_{\ell'=1}^{L-1}\exp S_{i,\ell'}}
\end{bmatrix}A_{i}\,.
\]

Next, we compute each quantity.
As before, we consider~\eqref{eq:defXYZ}, except now we define $X = u_\ell^\top \beta$.
This time, we have
\[
\begin{bmatrix}
\mu_{1}\\
\mu_{2} \\
\mu_{3}
\end{bmatrix}=\begin{bmatrix}
u_{\ell}^{\top}\mu\\[0.25em]
x_{i}^{\top}\mu\\[0.25em]
e_{j}^{\top}\mu
\end{bmatrix}\,,\qquad \begin{bmatrix}
\sigma_{1}^{2} & \rho_{1,2}\sigma_{1}\sigma_{2} & \rho_{1,3}\sigma_{1}\sigma_{3}\\
\rho_{1,2}\sigma_{1}\sigma_{2} & \sigma_{2}^{2} & \rho_{2,3}\sigma_{2}\sigma_{3}\\
\rho_{1,3}\sigma_{1}\sigma_{3} & \rho_{2,3}\sigma_{2}\sigma_{3} & \sigma_{3}^{2}
\end{bmatrix}=\begin{bmatrix}
u_{\ell}^{\top}\\[0.25em]
x_{i}^{\top}\\[0.25em]
e_{j}^{\top}
\end{bmatrix}\Sigma \begin{bmatrix}
u_{\ell} & x_{i} & e_{j}\end{bmatrix}\,.
\]
Then, as before, we have
\[
A_{i}=\mathbb{E}[\max\{ Y,0\} ]\,,\qquad F_{i}=\mathbb{P}(Y>0)\,,\qquad R_{i,j}=\mathbb{E}[Z\ind_{Y>0}]\,,
\]
and
\[
P_{i,j}=\mathbb{E}[\max\{ Y,0\} Z]\qquad\text{and}\qquad Q_{i,j,\ell}=\mathbb{E}[X\ind_{Y>0}Z]\,.
\]
The only new quantity to compute is
\begin{align*}
R_{i,j} & =\mathbb{E}[Z\ind_{Y>0}]=\mathbb{E}[(Z-\gamma Y)\ind_{Y>0}]+\mathbb{E}[\gamma Y\ind_{Y>0}]\\
 & =\mathbb{E}[Z-\gamma Y]\,\mathbb{P}(Y>0)+\gamma\,\mathbb{E}[\max\{ Y,0\} ]\\
 &= F_i M_{i,j}+\gamma A_{i}
\end{align*}
and
\[
S_{i,\ell}=\mathbb{E}[\omega_{\ell}\ReLU(\beta^\top x_i)]=\mathbb{E}[\ReLU(\beta^\top x_i) \,u_{\ell}^{\top}\beta]+v_{\ell}\,\mathbb{E}\ReLU(\beta^\top x_i)=C_{i}+v_{\ell}A_{i}\,.
\]

\subsection{Auxiliary Gaussian calculus\label{sec:auxiliary}}

Recall that $\phi(\cdot\mid \mu,\Sigma)$ is the density function of $\mathcal{N}(\mu,\Sigma)$.
Then we can compute the gradient and Hessian w.r.t.~the variable
$ \theta $:
\begin{align}
\nabla_{ \theta }\phi( \theta  \mid \mu,\Sigma)
& =\phi( \theta  \mid \mu,\Sigma)\,\nabla_{ \theta }\bigl(-\frac{1}{2}\, \theta ^{\top}\Sigma^{-1} \theta + \theta ^{\top}\Sigma^{-1}\mu-\frac{1}{2}\,\mu^{\top}\Sigma^{-1}\mu\bigr)\nonumber \\
 & =-\phi( \theta  \mid \mu,\Sigma)\,\Sigma^{-1}( \theta -\mu)\,,\label{eq:gaussian-eq-0}\\
\nabla_{ \theta }^{2}\phi( \theta  \mid \mu,\Sigma)
& =-\Sigma^{-1}( \theta -\mu)\,[\nabla_{ \theta }\phi( \theta  \mid \mu,\Sigma)]^{\top}-\phi( \theta  \mid \mu,\Sigma)\,\Sigma^{-1}\nonumber \\
 & =\phi( \theta  \mid \mu,\Sigma)\,\Sigma^{-1}( \theta -\mu)( \theta -\mu)^{\top}\Sigma^{-1}-\phi( \theta  \mid \mu,\Sigma)\,\Sigma^{-1}\,.\nonumber 
\end{align}
In addition, we can also compute the gradient w.r.t.~the parameters
$\mu$ and $\Sigma$:
\begin{align}
\nabla_{\mu}\phi( \theta  \mid \mu,\Sigma) & =\phi( \theta  \mid \mu,\Sigma)\,\nabla_{\mu}\bigl(-\frac{1}{2}\, \theta ^{\top}\Sigma^{-1} \theta + \theta ^{\top}\Sigma^{-1}\mu-\frac{1}{2}\,\mu^{\top}\Sigma^{-1}\mu\bigr)\nonumber \\
 & =\phi( \theta  \mid \mu,\Sigma)\,\Sigma^{-1}( \theta -\mu)\nonumber \\
 & =-\nabla_{ \theta }\phi( \theta  \mid \mu,\Sigma)\,,\label{eq:gaussian-eq-1}
\end{align}
and
\begin{align}
\nabla_{\Sigma}\phi( \theta  \mid \mu,\Sigma)
& =\frac{1}{\sqrt{(2\pi)^{d}\det \Sigma}}\nabla_{\Sigma}\exp\bigl[-\frac{1}{2}\,( \theta -\mu)^{\top}\Sigma^{-1}( \theta -\mu)\bigr]\nonumber \\
 & \qquad{} +\exp\bigl[-\frac{1}{2}\,( \theta -\mu)^{\top}\Sigma^{-1}( \theta -\mu)\bigr]\,\nabla_{\Sigma}\frac{1}{\sqrt{(2\pi)^{d}\det \Sigma}}\nonumber \\
 & =\frac{1}{2}\,\phi( \theta  \mid \mu,\Sigma)\,\Sigma^{-1}( \theta -\mu)( \theta -\mu)^{\top}\Sigma^{-1}-\frac{1}{2}\,\phi( \theta  \mid \mu,\Sigma)\,\frac{1}{\det \Sigma}\,\nabla_{\Sigma}\det \Sigma\nonumber \\
 & =\frac{1}{2}\,\phi( \theta  \mid \mu,\Sigma)\,\Sigma^{-1}( \theta -\mu)( \theta -\mu)^{\top}\Sigma^{-1}-\frac{1}{2}\,\phi( \theta  \mid \mu,\Sigma)\,\Sigma^{-1}\nonumber \\
 & =\frac{1}{2}\,\nabla_{ \theta }^{2}\phi( \theta  \mid \mu,\Sigma)\,.\label{eq:gaussian-eq-2}
\end{align}

\section{Experimental details}

The numerical experiments conducted in this paper are implemented with PyTorch in Python, using a 2023 MacBook Pro with Apple M2 Pro chip and 32GM memory. The fully-connected layers are implemented using PyTorch's built-in functions. The GM layers are implemented using the derivation in Appendix~\ref{app:derivations} (thanks to PyTorch's Automatic Differentiation engine, we only need to implement the loss function, and there is no need to implement the gradients explicitly).

The error bars in Figures~\ref{fig:GMM-test-error}, \ref{fig:NN-test-error}, \ref{fig:fix_beta}, \ref{fig:subsampling} and \ref{fig:two-layers} represent one standard error computed over 5 independent trials. When implementing a two-layer GM network, the output of the first GM layer is normalized to have unit norm, before being sent to the second layer as input.

\end{document}